\documentclass[11pt]{article}

\usepackage{amsmath,amssymb,amsthm}

\usepackage{times}
\usepackage{graphicx} % more modern
\usepackage{subfigure} 

\usepackage{algorithm}
\usepackage{algorithmic}

\usepackage{hyperref}

\usepackage[top=30truemm,bottom=30truemm,left=25truemm,right=25truemm]{geometry}
\usepackage{bbm}

\usepackage[utf8]{inputenc} % allow utf-8 input
\usepackage[T1]{fontenc}    % use 8-bit T1 fonts
\usepackage{hyperref}       % hyperlinks
\usepackage{url}            % simple URL typesetting
\usepackage{booktabs}       % professional-quality tables
\usepackage{amsfonts}       % blackboard math symbols
\usepackage{nicefrac}       % compact symbols for 1/2, etc.
\usepackage{microtype}      % microtypography
\usepackage{graphicx}

\usepackage{color,comment}

\newcommand{\R}{\mathbb{R}}

\newcommand{\mG}{\mathcal{G}}

\newcommand{\mP}{\mathcal{P}}

\newcommand{\mX}{\mathcal{X}}

\newcommand{\mE}{\mathcal{E}}
\renewcommand{\hat}{\widehat}
\renewcommand{\tilde}{\widetilde}

\newtheorem{theorem}{Theorem}
\newtheorem{assumption}[theorem]{Assumption}

\newtheorem{lemma}[theorem]{Lemma}

\newcommand{\etal}{\textit{et al.}}
\usepackage{authblk}

\newcommand{\bc}{\textcolor{blue}}

\begin{document}

\title{On Tensor Train Rank Minimization: Statistical Efficiency and Scalable Algorithm}

\author[1]{Masaaki Imaizumi}
\author[2]{Takanori Maehara}
\author[2,3]{Kohei Hayashi}
\affil[1]{Institute of Statistical Mathematics}
\affil[2]{RIKEN Center for Advanced Intelligence Project}
\affil[3]{National Institute of Advanced Industrial Science and Technology}

\date{}

\maketitle

\begin{abstract}
  Tensor train (TT) decomposition provides a space-efficient
  representation for higher-order tensors. Despite its advantage, we
  face two crucial limitations when we apply the TT decomposition to
  machine learning problems: the lack of statistical theory and of
  scalable algorithms. In this paper, we address the
  limitations. First, we introduce a convex relaxation of the TT
  decomposition problem and derive its error bound for the tensor
  completion task. Next, we develop an alternating optimization method
  with a randomization technique, in which the time complexity is as
  efficient as the space complexity is. In experiments, we numerically
  confirm the derived bounds and empirically demonstrate the
  performance of our method with a real higher-order tensor.
\end{abstract}

%%% Local Variables:
%%% mode: latex
%%% TeX-master: "TTcomp_NIPS2017.tex"
%%% End:

\section{Introduction}

Tensor decomposition is an essential tool for dealing with data
represented as multidimensional arrays, or simply, tensors. Through
tensor decomposition, we can determine latent factors of an input
tensor in a low-dimensional multilinear space, which saves the storage
cost and enables predicting missing elements. Note that, a different
multilinear interaction among latent factors defines a different
tensor decomposition model, which yields a ton of variations of tensor
decomposition. For general purposes, however, either Tucker
decomposition~\cite{tucker1966some} or \emph{CANDECOMP/PARAFAC (CP)
  decomposition}~\cite{harshman1970foundations} model is commonly
used.

In the past three years, an alternative tensor decomposition model,
called \emph{tensor train (TT)}
decomposition~\cite{oseledets2011tensor} has actively been studied in
machine learning communities for such as approximating the inference
on a Markov random field~\cite{novikov2014putting}, modeling
supervised learning~\cite{novikov2016exponential, NIPS2016_6211},
analyzing restricted Boltzmann machine~\cite{chen2017equivalence}, and
compressing deep neural networks~\cite{novikov2015tensorizing}.
%~\cite{novikov2014putting,novikov2015tensorizing,novikov2016exponential,NIPS2016_6211,chen2017equivalence}.
A key property is that, for higher-order tensors, TT decomposition
provides more space-saving representation called TT format while
preserving the representation power. Given an order-$K$ tensor (i.e., a
$K$-dimensional tensor), the space complexity of Tucker decomposition
is exponential in $K$, whereas that of TT decomposition is linear
in $K$. Further, on TT format, several mathematical
operations including the basic linear algebra operations can be performed
efficiently~\cite{oseledets2011tensor}.
%TT
%decomposition is now actively used in machine learning context such as
%modeling supervised
%learning~\cite{novikov2016exponential,NIPS2016_6211}, analyzing
%restricted Bolzmann machine~\cite{chen2017equivalence}, and deep
%neural networks~\cite{novikov2015tensorizing}.

%
%To compute TT decomposition, several optimization methods have been
%proposed. For example, tensor compression (i.e. tensor elements are
%fully observed) via iterative singular value decomposition
%(SVD)~\cite{oseledets2010tt} and tensor completion (i.e. tensor
%elements are partially missing) via alternating
%minimization~\cite{grasedyck2015alternating,wang2016tensor} or convex
%formulation~\cite{phien2016efficient}. \knote{Edit after confirming the time complexity of our methods}

Despite its potential importance, we face two crucial limitations when
applying this decomposition to a much wider class of machine learning problems.
First, its statistical performance is unknown. In Tucker decomposition
and its variants, many authors addressed the generalization error and
derived statistical bounds
(e.g. \cite{tomioka2011statistical,tomioka2013convex}). For example,
Tomioka \etal \cite{tomioka2011statistical} clarify the way in which using the convex
relaxation of Tucker decomposition, the generalization error is
affected by the rank (i.e., the dimensionalities of latent factors),
 dimension of an input, and number of observed elements. In
contrast, such a relationship has not been studied for TT decomposition
yet.
Second, standard TT decomposition algorithms, such as alternating
least squares (ALS)~\cite{grasedyck2015alternating,wang2016tensor} ,
require a huge computational cost. The main bottleneck arises from the
singular value decomposition (SVD) operation to an ``unfolding''
matrix, which is reshaped from the input tensor. The size of the
unfolding matrix is huge and the computational cost grows
exponentially in $K$. 

%Second, there are many rank parameters that are not easy to
%determine. This could be done by using a regularizer that surrogates
%the sum of ranks (e.g. \cite{tomioka2011statistical}). However, this
%convex relaxation incurs further computational cost, which is
%unrealistic to compute \knote{Show actual computational cost}.

In this paper, we tackle the above issues and present a scalable yet
statistically-guaranteed TT decomposition method. We first introduce a
convex relaxation of the TT decomposition problem and its optimization
algorithm via the alternating direction method of multipliers (ADMM).
Based on this, a statistical error bound for tensor completion is
derived, which achieves the same statistical efficiency as the convex
version of Tucker decomposition does. Next, because the ADMM algorithm
is not sufficiently scalable, we develop an alternative method by using a
randomization technique. At the expense of losing the global
convergence property, the dependency of $K$ on the time complexity is
reduced from exponential to quadratic. In addition, we show that a
similar error bound is still guaranteed.
In experiments, we numerically confirm the derived bounds and
empirically demonstrate the performance of our method using a real
higher-order tensor.

%%% Local Variables:
%%% mode: latex
%%% TeX-master: "TTcomp_NIPS2017.tex"
%%% End:

\section{Preliminaries}

\subsection{Notation}

Let $\mX \subset \R^{I_1 \times \cdots \times I_K}$ be the space of
order-$K$ tensors, where $I_k$ denotes the dimensionality of the $k$-th
mode for $k=1,\dots,K$.  For brevity, we define
$I_{<k} := \prod_{k'<k}I_{k'}$; similarly, $I_{\leq k}, I_{k<}$ and
$I_{k \leq}$ are defined.  For a vector $Y \in \R^d$, $[Y]_i$ denotes
the $i$-th element of $Y$.  Similarly, $[X]_{i_1,\ldots,i_K}$ denotes
the $(i_1,\ldots,i_K)$ elements of a tensor $X\in\mX$. Let
$[X]_{i_1,\ldots,i_{k-1},:,i_{k+1},\ldots,i_K}$ denote an
$I_k$-dimensional vector
$(X_{i_1,\ldots,i_{k-1},j,i_{k+1},\ldots,i_K})_{j=1}^{I_k}$ called the
mode-$k$ fiber.  For a vector $Y \in \R^d$, $\|Y\| = (Y^T Y)^{1/2}$
denotes the $\ell_2$-norm and $\|Y\|_{\infty} = \max_i|[Y]_i|$ denotes
the max norm.  For tensors $X,X' \in \mX$, an inner product is defined
as
$\langle X,X' \rangle := \sum_{i_1,\ldots,i_K =1}^{I_1 \dots I_K}
X(i_1,\ldots,i_K)X'(i_1,\ldots,i_K)$
and $\|X\|_{F} = \langle X,X \rangle^{1/2}$ denotes the Frobenius
norm.  For a matrix $Z$, $\|Z\|_s := \sum_{j} \sigma_{j}(Z)$ denotes
the Schatten-1 norm, where $\sigma_j(\cdot)$ is a $j$-th singular value
of $Z$.

\subsection{Tensor Train Decomposition}

%\textit{Tensor train (TT) decomposition} is a tensor factorization
%method with a matrix product representation
%\cite{oseledets2010tt,oseledets2011tensor}.  
Let us define a tuple of positive integers $(R_1, \ldots, R_{K-1})$
and an order-$3$ tensor $G_k \in \R^{I_k \times R_{k-1} \times R_k}$
for each $k = 1,\ldots,K$.  Here, we set $R_0 = R_K = 1$.  Then, TT
decomposition represents each element of $X$ as follows:
\begin{align}
	X_{i_1,\ldots,i_K} = [G_1]_{i_1,:,:} [G_2]_{i_2,:,:} \cdots [G_K]_{i_K,:,:}. \label{eq:tt}
\end{align}
Note that $[G_k]_{i_k,:,:}$ is an $R_{k-1} \times R_k$ matrix.  We
define $\mG := \{G_k\}_{k=1}^K$ as a set of the tensors, and let $X(\mG)$
be a tensor whose elements are represented by $\mG$ as
\eqref{eq:tt}.  The tuple $(R_1, \ldots, R_{K-1})$ controls
the complexity of TT decomposition, and it is called a \textit{Tensor
  Train (TT) rank}.  Note that TT decomposition is universal, i.e.,
any tensor can be represented by TT decomposition with sufficiently
large TT rank~\cite{oseledets2010tt}.

When we evaluate the computational complexity, we assume the shape of
$\mG$ is roughly symmetric. That is, we assume there exist
$I,R\in\mathbb{N}$ such that $I_k=O(I)$ for $k=1,\dots,K$ and
$R_k=O(R)$ for $k=1,\dots,K-1$.

\subsection{Tensor Completion Problem}

Suppose there exists a true tensor $X^* \in \mX$ that is unknown, and
a part of the elements of $X^*$ is observed with some noise.  Let
$S \subset \{(j_1,j_2,
\ldots,j_K)\}_{j_1,\ldots,j_K=1}^{I_1,\ldots,I_K}$
be a set of indexes of the observed elements and
$n := |S| \leq \prod_{k=1}^K I_k$ be the number of observations.  Let
$j(i)$ be an $i$-th element of $S$ for $i=1,\ldots,n$, and $y_i$
denote $i$-th observation from $X^*$ with noise.  We consider the
following observation model:
\begin{align}
	y_i = [X^*]_{j(i)} + \epsilon_i, \label{model:obs}
\end{align}
where $\epsilon_i$ is i.i.d. noise with zero mean and variance
$\sigma^2$.  For simplicity, we introduce  observation vector
$Y := (y_1, \ldots, y_n)$, noise vector
$\mE := (\epsilon_1, \ldots , \epsilon_n)$, and rearranging operator
$\mathfrak{X} : \mX \to \mathbb{R}^n$ that randomly picks the elements of $X$.
%  $[\mathfrak{X}(X)]_i = [X]_{j(i)}$.
Then, the model \eqref{model:obs} is rewritten as follows:
\begin{align*}
	Y = \mathfrak{X}(X^*) + \mE.
\end{align*}

%%%
The goal of tensor completion is to estimate the true tensor $X^*$
from the observation vector $Y$.  Because the estimation problem is
ill-posed, we need to restrict the degree of freedom of $X^*$, such as
rank. Because the direct optimization of rank is difficult, its convex
surrogation is alternatively
used~\cite{candes2012exact,candes2010matrix, krishnamurthy2013low,
  zhang2016exact, phien2016efficient}.  For tensor
completion, the convex surrogation yields the following optimization
problem
\cite{gandy2011tensor,liu2013tensor,signoretto2011tensor,tomioka2010estimation}:
\begin{align}
	\min_{X \in \Theta} \left[ \frac{1}{2n} \|Y - \mathfrak{X}(X)\|^2 + \lambda_n \|X\|_{s^*} \right], \label{opt:general}
\end{align}
where $\Theta \subset \mX$ is a convex subset of $\mX$, 
%and
%$\Omega : \Theta \to \R_+$ is a regularization for tensors, 
$\lambda_n\geq 0$ is a regularization coefficient, and
$ \|\cdot\|_{s^*}$ is the overlapped Schatten norm defined as
$ \|X\|_{s^*} := \frac{1}{K} \sum_{k=1}^K \|\tilde{X}_{(k)}\|_s$.
Here, $\tilde{X}_{(k)}$ is the $k$-unfolding matrix defined by
concatenating the mode-$k$ fibers of $X$.  The overlapped Schatten
norm regularizes the rank of $X$ in terms of Tucker
decomposition~\cite{negahban2011estimation, tomioka2011statistical}.
Although the Tucker rank of $X^*$ is unknown in general, the convex
optimization adjusts the rank depending on $\lambda_n$.

To solve the convex problem~\eqref{opt:general}, the ADMM algorithm is often
employed~\cite{boyd2011distributed,tomioka2010estimation,
  tomioka2011statistical}.  Since the overlapped Schatten norm is not
differentiable, the ADMM algorithm avoids the differentiation of the
regularization term by alternatively minimizing the augmented
Lagrangian function iteratively.

%%% Local Variables:
%%% mode: latex
%%% TeX-master: "TTcomp_NIPS2017.tex"
%%% End:

\section{Convex Formulation of TT Rank Minimization}

%\knote{Replace TT-rank by TT rank}

To adopt TT decomposition to the convex optimization problem as
\eqref{opt:general}, we need the convex surrogation of TT rank.  For
that purpose, we introduce the \textit{Schatten TT
  norm}~\cite{phien2016efficient} as follows:
\begin{align}
	\|X\|_{s,T} := \frac{1}{K-1} \sum_{k=1}^{K-1} \|Q_k(X)\|_{s} := \frac{1}{K-1}\sum_{k=1}^{K-1} \sum_{j} \sigma_j(Q_k(X)), \label{def:ttnorm}
\end{align}
where $Q_k : \mX \to \R^{I_{\leq k} \times I_{k<}}$ is a reshaping
operator that converts a tensor to a large matrix where the first $k$
modes are combined into the rows and the rest $K-k$ modes are combined
into the columns. Oseledets \etal \cite{oseledets2011tensor} shows that the matrix
rank of $Q_k(X)$ can bound the $k$-th TT rank of $X$, implying that the
Schatten TT norm surrogates the sum of the TT rank. Putting the
Schatten TT norm into \eqref{opt:general}, we obtain the following
optimization problem:
\begin{align}
	\min_{X \in \mX} \left[ \frac{1}{2n}\| Y - \mathfrak{X}(X) \|^2 + \lambda_n\|X\|_{s,T} \right]. \label{eq:prob}
\end{align}
%Here, $\lambda_n$ is a penalty coefficient.

\subsection{ADMM Algorithm}

To solve \eqref{eq:prob}, we consider the
augmented Lagrangian function
$L(x, \{Z_k\}_{k=1}^{K-1}, \{\alpha_k\}_{k=1}^{K-1})$,
%as
%\begin{align*}
%	&L(x, \{Z_k\}_{k=1}^{K-1}, \{\alpha_k\}_{k=1}^{K-1}) \\
%	&= \frac{1}{n}\|\tilde {\Omega}x - Y\|_F^2 + \lambda_n \sum_{k=1}^{K-1} \|Z_k\|_s  + \sum_{k=1}^{K-1} \left( \eta \alpha_k^T (x - \mbox{vec}Z_k) + \frac{\eta}{2}\|x - \mbox{vec}Z_k\|_F^2 \right),
%\end{align*}
where $x \in \mathbb{R}^{\prod_k I_k}$ is the
vectorization of $X$, $Z_k$ is a reshaped matrices with size
$I_{\leq k} \times I_{k<}$, and
$\alpha_k \in \mathbb{R}^{\prod_k I_k}$ is a coefficient for
constraints.  Given initial points
$(x^{(0)}, \{Z_k^{(0)}\}_k, \{\alpha_k^{(0)}\}_k)$, the $\ell$-th step
of ADMM is written as follows:
\begin{align*}
	&x^{(\ell + 1)} = \left( \tilde{\Omega}^T Y+ n \eta \frac{1}{K-1}\sum_{k=1}^{K-1} (V_k(Z_k^{(\ell)}) - \alpha_k^{(\ell)})\right) / (1 + n \eta K),\\
	&Z_k^{(\ell + 1)} = \mbox{prox}_{\lambda_n / \eta} (V_k^{-1}(x^{(\ell + 1)} + \alpha_k^{(\ell)})), ~~ k = 1,\ldots,K,\\
	& \alpha_k^{(\ell + 1)} = \alpha_k^{(\ell)} + (x^{(\ell + 1)} - V_k(Z_k^{(\ell + 1)})), ~~  k = 1,\ldots,K.
\end{align*}
Here, $\tilde{\Omega}$ is an $n \times \prod_{k=1}^{I_k}$ matrix that
works as the inversion mapping of $\mathfrak{X}$; 
%$./$ denotes an element-wise division; 
$V_k$ is a vectorizing operator of an $I_{\leq k} \times I_{k<}$
matrix; $\mbox{prox}(\cdot)$ is the shrinkage operation of the
singular values as $\mbox{prox}_{b}(W) = U \max\{S-bI,0\}V^T$, where
$USV^T$ is the singular value decomposition of $W$; $\eta>0$ is a
hyperparameter for a step size.  We stop the iteration when the
convergence criterion is satisfied (e.g. as suggested by
Tomioka \etal \cite{tomioka2011statistical}). Since the Schatten TT
norm~\eqref{def:ttnorm} is convex, the sequence of the variables of
ADMM is guaranteed to converge to the optimal solution (Theorem
5.1, \cite{gandy2011tensor}). We refer to this algorithm as \emph{TT-ADMM}.

TT-ADMM requires huge resources in terms of both time and space.  For
the time complexity, the proximal operation of the Schatten TT norm,
namely the SVD thresholding of $V_k^{-1}$, yields the dominant, which is
$O(I^{3K/2})$ time. For the space complexity, we have $O(K)$ variables
of size $O(I^K)$, which requires $O(KI^K)$ space.

%%% Local Variables:
%%% mode: latex
%%% TeX-master: "TTcomp_NIPS2017.tex"
%%% End:

\section{Alternating Minimization with Randomization}

We first consider reducing the space complexity of the ADMM
approach. The idea is simple: we only maintain the TT format of the
input tensor rather than the input tensor itself.  This leads the
following optimization problem:
\begin{align}
    \min_{\mG} \left[ \frac{1}{2n} \|Y - \mathfrak{X}(X(\mG))\|^2 + \lambda_n \| X(\mG)\|_{s,T} \right]. \label{opt:als0}
\end{align}
Remember that $\mG=\{G_k\}_k$ is the set of TT components and $X(\mG)$
is the tensor given by the TT format with $\mG$. Now we only need to
store the TT components $\mG$, which drastically improves the space
efficiency.

\subsection{Randomized TT Schatten norm}

Next, we approximate the optimization of the Schatten TT norm. To avoid
the computation of exponentially large-scale SVDs in the Schatten TT norm,
we employ a technique called the ``very sparse random
projection''~\cite{li2006very}. The main idea is that, if the size of a
matrix is sufficiently larger than its rank, then its singular values
(and vectors) are well preserved even after the projection by a sparse
random matrix. This motivates us to use the Schatten TT norm over the
random projection.

Preliminary, we introduce tensors for the random projection. Let
$D_1,D_2 \in \mathbb{N}$ be the size of the matrix after projection.
For each $k = 1,\ldots,K-1$ and parameters, let
$\Pi_{k,1} \in \R^{D_1 \times I_1 \times \cdots \times I_{k}}$ be a
tensor whose elements are independently and identically distributed as follows:
\begin{align}
    [\Pi_{k,1}]_{d_1,i_1, \ldots,i_k} = 
    \begin{cases}
        +\sqrt{s/d_1} \quad & \mbox{~with probability~}1/2s,\\
        0 \quad & \mbox{~with probability~}1-1/s,\\
        -\sqrt{s/d_1} \quad & \mbox{~with probability~}1/2s,        \label{def:pi}
    \end{cases}
\end{align}
for $i_1,\ldots,i_k$ and $d_1 = 1,\ldots,D_1$.
Here, $s>0$ is a hyperparameter controlling sparsity.
Similarly, we introduce a tensor
$\Pi_{k,2} \in \R^{D_2 \times I_{k+1} \times \cdots \times I_{K-1}}$ that is defined in the same way as $\Pi_{k,1}$.
With $\Pi_{k,1}$ and $\Pi_{k,2}$, let
$\mP_k : \mX \to \R^{D_1 \times D_2}$ be a random projection operator
whose element is defined as follows:
\begin{align}
    [\mP_k(X)]_{d_1,d_2} 
  &= \sum_{j_1=1}^{I_1}  \cdots \sum_{j_{K}=1}^{I_{K}} [\Pi_{k,1}]_{d_1,j_1,\ldots,j_{k}}  [X]_{j_1,\ldots,j_K}[\Pi_{k,2}]_{d_2,j_{k+1},\ldots,j_{K}}. \label{eq:random1} 
%  &=  \sum_{(j_1, \ldots, j_k) \in \pi^{(k)}_{1}} \sum_{(j_{k+1}, \ldots, j_K) \in \pi^{(k)}_{2}} [\Pi_{k,1}]_{d_1,j_1,\ldots,j_{k}}  [X]_{j_1,\ldots,j_K}[\Pi_{k,2}]_{d_2,j_{k+1},\ldots,j_{K}}. \label{eq:random1} 
%\\
%  &= \sum_{j_1=1}^{I_1}  \cdots \sum_{j_{K}=1}^{I_{K}} [\Pi_{k,1}]_{d_1,j_1,\ldots,j_{k}} [G_1]_{j_1}  \cdots [G_K]_{j_K} [\Pi_{k,2}]_{d_2,j_{k+1},\ldots,j_{K}}. \label{eq:random2}
\end{align}
%In the second line we used the fact that $X$ has the TT format
%$X = X(\mG)$. 
Note that we can compute the above projection by using the facts that
$X$ has the TT format and the projection matrices are sparse.  Let
$\pi^{(k)}_{j}$ be a set of indexes of non-zero elements of
$\Pi_{k,j}$. Then, using the TT representation of $X$,
\eqref{eq:random1} is rewritten as
\begin{align*}
  [\mP_k(X(\mG))]_{d_1,d_2}  
  =& \sum_{(j_1, \ldots, j_k) \in \pi^{(k)}_{1}} [\Pi_{k,1}]_{d_1,j_1,\ldots,j_{k}} [G_1]_{j_1}  \cdots [G_k]_{j_k}
\\
  &\sum_{(j_{k+1}, \ldots, j_K) \in \pi^{(k)}_{2}} [G_k]_{j_{k+1}} \cdots [G_K]_{j_K} [\Pi_{k,2}]_{d_2,j_{k+1},\ldots,j_{K}},
\end{align*}
If the projection matrices have only $S$ nonzero elements (i.e.,
$S = |\pi^{(1)}_j|= |\pi^{(2)}_j|$), the computational cost of the
above equation is $O(D_1D_2SKR^3)$.

The next theorem guarantees that the Schatten-1 norm of $\mP_k(X)$
approximates the original one.
\begin{theorem} \label{thm:random} Suppose $X\in\mX$ has TT rank
  $(R_1, \ldots,R_k)$.  Consider the reshaping operator $Q_k$ in
  \eqref{def:ttnorm}, and the random operator $\mP_k$ as
  \eqref{eq:random1} with tensors $\Pi_{k,1}$ and $\Pi_{k,2}$ defined
  as \eqref{def:pi}.  
  If $D_1, D_2 \ge \max\{ R_k, 4 (\log(6 R_k) + \log(1 / \epsilon))/\epsilon^2 \}$, and all the singular vectors $u$ of $Q(X)_k$ are well-spread as
  $ \sum_j |u_j|^3 \le \epsilon/(1.6 k \sqrt{s})$, we have
  \begin{align*}
      \frac{1-\epsilon}{R_k} \|Q_k(X)\|_{s} \leq \| \mP_k (X)\|_{s} \leq (1+\epsilon) \|Q_k(X)\|_{s},
  \end{align*}
  with probability at least $1 - \epsilon$.

\if0
    \textcolor{red}{ [OLD]
    Then, for any $k$ and $\epsilon > 0$, with probability at least [TBW] with a constant $c_{p} > 0$, the following relation holds
    \begin{align*}
        (1-\epsilon)^{1/2}\|Q_k(X)\|_{s} \leq \| \mP_k (X)\|_{s} \leq (1+\epsilon)^{1/2}\|Q_k(X)\|_{s},
    \end{align*}
    when $D_1 \geq R_k$ and $D_2 \geq R_k$ are satisfied.
    }
    \fi
\end{theorem}

Note that the well-spread condition can be seen as a stronger version of the incoherence assumption which will be discussed later.
%since $\sum_i |u_i|^3 \le C$ for all singular vector $u$ implies $\| P_U e_i \|_2 C^{1/4}$ for all $i$.

\subsection{Alternating Minimization}

%Preliminarily, we consider $X$ with TT decomposition as a function of $G_k$. 
%Namely, we fix $\mG \backslash \{G_k\}$ and define a tensor as
%\begin{align*}
%    X_k(\mG) := X(G_k ; \mG \backslash \{G_k\}).
%\end{align*}
%Then, the optimization problem \eqref{opt:als} is rewritten as
%\begin{align}
%    \min_{G_k} \left[ \frac{1}{2n}\|Y - \mathfrak{X}(X_k(\mG))\|^2 + \lambda_n \| X_k(\mG)\|_{s,T} \right], k = 1,\ldots, K. \label{opt:als}
%\end{align}
%By the setting of TT decomposition, $G_k$ and it is a $3$-way tensor with size $I_k \times R_{k-1} \times R_k$.
%Thus, we require only a small size of computational memory to handle $G_k$ as the control variable.

% $\mG \backslash \{G_k\} := \{G_{k'}\}_{k'\not=k}$

Note that the new problem \eqref{opt:als0} is non-convex because
$X(\mG)$ does not form a convex set on $\mX$.  However, if we fix
$\mG$ except for $G_k$, it becomes convex with respect to
$G_k$. Combining with the random projection, we obtain the following
minimization problem:
\begin{align}
    \min_{G_k} \left[ \frac{1}{2n} \|Y - \mathfrak{X} (X(\mG))\|^2 + \frac{\lambda_n}{K-1} \sum_{k'=1}^{K-1} \|\mP_{k'}(X (\mG))\|_{s} \right]. \label{opt:als2}
\end{align}
We solve this by the ADMM method for each $k = 1,\ldots,K$.  Let
$g_k \in \R^{I_kR_{k-1}R_k}$ be the vectorization of $G_k$, and
$W_{k'}\in \R^{D_1 \times D_2}$ be a matrix for the randomly projected matrix.
The augmented Lagrangian function is then given by
$L_k(g_k, \{W_{k'}\}_{k'=1}^{K-1}, \{\beta_{k'}\}_{k'=1}^{K-1})$,
%\begin{align*}
%    &L_k(g_k, \{W_{k'}\}_{k'=1}^{K-1}, \{\beta_{k'}\}_{k'=1}^{K-1}) \\
%    &= \frac{1}{2n} \|Y - \Omega g_k \|^2 + \lambda_n\sum_{k'=1}^{K-1}  \|Z_{k'}\|_{s^*}  + \sum_{k'=1}^{K-1} \left\{ \beta_{k'}^{T} (\Gamma_{k'} g_k - \mbox{vec}Z_{k'}) + \frac{1}{2} \|\Gamma_{k'} g_k - \mbox{vec}Z_{k'}\|^2 \right\},
%\end{align*}
where $\{\beta_{k'} \in \R^{D_1D_2} \}_{k'=1}^{K-1}$ are the Lagrange
multipliers.  Starting from initial points
$(g_k^{(0)}, \{W_{k'}^{(0)}\}_{k'=1}^{K-1},
\{\beta_{k'}^{(0)}\}_{k'=1}^{K-1})$,
the $\ell$-th ADMM step is written as follows:
\begin{align*}
    &g_k^{(\ell + 1)} = \left( \Omega^T \Omega / n + \eta \sum_{k'=1}^{K-1} \Gamma_{k'}^{T}\Gamma_{k'} \right)^{-1}  \left( \Omega^T Y / n + \frac{1}{K-1}\sum_{k'=1}^{K-1} \Gamma_{k'}^{T}( \eta \tilde{V}_k(W_{k'}^{(\ell)}) - \beta_{k'}^{(\ell)} )  \right) ,\\
    &W_{k'}^{(\ell + 1)} = \mbox{prox}_{\lambda_n / \eta} \left( \tilde{V}_k^{-1} ( \Gamma_{k'} g_k^{(\ell+1)} +\beta_{k'}^{(\ell)} ) \right),~~k'=1,\ldots,K-1, \\
    &  \beta_{k'}^{(\ell + 1)} = \beta_{k'}^{(\ell)} + (\Gamma_{k'} g_k^{(\ell + 1)} - \tilde{V}_k(W_{k'}^{(\ell + 1)})),~~k'=1,\ldots,K-1.
\end{align*}
Here, $\Gamma^{(k)} \in \R^{D_1 D_2 \times I_kR_{k-1}R_k}$ is the matrix
imitating the mapping
$G_k\mapsto\mP_k(X(G_k;\mG \backslash \{G_k\}))$, $\tilde{V}_k$ is a vectorizing operator of $D_1 \times D_2$
matrix, and $\Omega $ is an $n \times I_kR_{k-1}R_k$
matrix of the operator
$\mathfrak{X} \circ X( \cdot ; \mG \backslash \{G_k\})$ with respect
to $g_k$.
Similarly to the convex approach, we iterate the ADMM steps until
convergence. We refer to this algorithm as \emph{TT-RALS}, where RALS stands for randomized ALS.

%The time and space complexities are as follows. To update
%$g_k^{(\ell)}$, we need $O(I^3R^6)$ time for the inversion of the
%$IR^2 \times IR^2$ matrix. The SVD thresholding requires $O(D^3SR^3)$
%time.
%
%$O(nK^2IR^2)$

The time complexity of TT-RALS at the $\ell$-th iteration is
$O((n + KD^2)KI^2R^4)$; the details are deferred to
Supplementary material.
%
%, computing the parameters requires $O(K(n+KD^2)I^2R^4)$,
%inverting an $IR^2 \times IR^2$ matrix for updating $g_k^{(\ell)}$ for
%all $k$ requires $O(KI^3R^6)$, and conducting the random projection
%and thresholding its requires $O(D^3SKR^3)$ times.  
%
The space complexity is $O(n + KI^2R^4)$, where $O(n)$ is for $Y$ and
$O(KI^2R^4)$ is for the parameters.

%\knote{Describe time and space complexities of TT-RALS}

%%% Local Variables:
%%% mode: latex
%%% TeX-master: "TTcomp_NIPS2017.tex"
%%% End:

\section{Theoretical Analysis}

In this section, we investigate how the TT rank and the number of
observations affect to the estimation error. Note that all the proofs
of this section are deferred to Supplementary material.

\subsection{Convex Solution}

To analyze the statistical error of the convex problem~\eqref{eq:prob}, we assume the
\textit{incoherence} of the reshaped version of $X^*$.
\begin{assumption}{(Incoherence Assumption)}\label{asmp:incoherence_convex}
There exists $k \in \{1,\ldots,K\}$ such that a matrix $Q_k(X^*)$ has orthogonal singular vectors $\{ u_r \in \R^{I_{\leq k}}, v_r \in  \R^{I_{k<}}\}_{r = 1}^{R_k}$ satisfying
	\begin{align*}
          \max_{1 \leq i \leq I_{<k}}\|P_{U}e_{i}\| \leq ( \mu R_k / I_{\leq k})^{\frac{1}{2}}
          \quad \text{and} \quad
          \max_{1 \leq i \leq I_{<k}}\|P_{V}e_{i}\| \leq ( \mu R_k / I_{k<})^{\frac{1}{2}}
	\end{align*}
	with some $0 \leq \mu < 1$.
	Here, $P_U$ and $P_V$ are linear projections onto spaces spanned by $\{u_r\}_r$ and $\{v_r\}_r$; $\{e_i\}_i$ is the natural basis.
\end{assumption}
Intuitively, the incoherence assumption requires that the singular
vectors for the matrix $Q_k(X^*)$ are well separated.
This type of assumption is commonly used in the matrix and tensor
completion studies~\cite{candes2012exact,candes2010matrix,
  zhang2016exact}.
Under the incoherence assumption, the error rate of the solution of \eqref{eq:prob} is derived.
\begin{theorem} \label{thm:convex} Let
  $X^*\in\mX$ be a true tensor with TT rank $(R_1,\dots,R_{K-1})$, and let
  $\hat{X}\in\mX$ be the minimizer of \eqref{opt:general}.  Suppose that
  $\lambda_n \geq \|\mathfrak{X}^*(\mE)\|_{\infty}/n$ and that Assumption
  \ref{asmp:incoherence_convex} for some $k' \in \{1,2,\ldots,K\}$ is
  satisfied.  If
	\begin{align*}
		n \geq C_{m'} \mu_{k'}^2 \max\{I_{\leq k'}, I_{k'<}\} R_{k'} \log^3  \max\{I_{\leq k'}, I_{k'<}\}
	\end{align*}
	with a constant $C_{m'}$, then with probability at least $1-( \max\{I_{\leq k'}, I_{k'<}\})^{-3}$ and with a constant $C_X$,
	\begin{align*}
		\|\hat{X} - X^*\|_F \leq  C_{X} \frac{\lambda_n}{K}\sum_{k=1}^{K-1} \sqrt{R_k }.
	\end{align*}
\end{theorem}

Theorem \ref{thm:convex} states that the bound for the statistical
error gets larger as the TT rank increases.  In other words,
completing a tensor is relatively easy as long as the tensor has small TT
rank.  Also, when we set $\lambda_n \to 0$ as $n$ increases, we can
state the consistency of the minimizer.

The result of Theorem \ref{thm:convex} is similar to that obtained from the studies on
matrix completion \cite{candes2010matrix, negahban2011estimation} and
tensor completion with the Tucker decomposition or SVD
\cite{tomioka2011statistical, zhang2016exact}.  Note that, although
Theorem~\ref{thm:convex} is for tensor completion, the result can
easily be generalized to other settings such as the tensor recovery or
the compressed sensing problems.

\subsection{TT-RALS Solution} \label{sec:theory}

%Let $\mG^{t}$ be the set of $G_k$ after repeating the TT-RALS steps
%for $t$ times.  When the update of $\mG^t$ is sufficiently small, we
%define $\hat{\mG} := \mG^t$ as the estimator of $\mG^*$ and
%$X(\hat{\mG})$ as the estimator of $X^*$ by TT-RALS.

Prior to the analysis, let $\mG^*$ be the true TT components such that
$X^* = X(\mG^*)$.  For simplification, we assume that the elements of
$\mG^*$ are normalized, i.e., $\|G_k\|=1, \forall k$, and an
$R_k \times I_{k-1}I_k$ matrix reshaped from $G_k^*$ has a $R_k$ row
rank.

In addition to the incoherence property (Assumption \ref{asmp:incoherence_convex}), we introduce an additional assumption on the initial point of the ALS iteration.

\begin{assumption}{(Initial Point Assumption)}\label{asmp:intial}
    Let $\mG^{\text{init}} := \{G_k^{\text{init}}\}_{k=1}^K$ be the initial point of the ALS iteration procedure.
    Then, there exists a finite constant $C_{\gamma}$ that satisfies
    \begin{align*}
        \max_{k \in \{1,\ldots,K\}} \|G_k^{\text{init}} - G_k^*\|_F \leq C_{\gamma}.
    \end{align*}
\end{assumption}

Assumption~\ref{asmp:intial} requires that the initial point is
sufficiently close to the true solutions $\mG^*$.  Although the ALS
method is not guaranteed to converge to the global optimum in general,
Assumption~\ref{asmp:intial} guarantees the
convergence to the true solutions~\cite{suzuki2016minimax}. Now we can evaluate the error
rate of the solution obtained by TT-RALS.
%This assumption is used in
%the field of ALS \cite{suzuki2016minimax}.

\begin{theorem} \label{thm:als} Let $X(\mG^*)$ be the true tensor
  generated by $\mG^*$ with TT rank $(R_1,\dots,R_{K-1})$, and
  $\hat{\mG} = \mG^t$ be the solution of TT-RALS at the $t$-th 
  iteration.  Further, suppose that Assumption~\ref{asmp:incoherence_convex} for some $k' \in \{1,2,\ldots,K\}$ and
  Assumption~\ref{asmp:intial} are satisfied, and suppose that Theorem \eqref{thm:random} holds with $\epsilon > 0$ for $k=1,\dots,K$.  Let $C_m,C_A,C_B > 0$ be $0 < \chi < 1$ be some
  constants.  If
    \begin{align*}
        n \geq C_{m} \mu_{k'}^2 R_{k'} \max\{I_{\leq k'}, I_{k'<}\} \log^3  \max\{I_{\leq k'}, I_{k'<}\},
    \end{align*}
    and the number of iterations $t$ satisfies
        $t \geq (\log \chi)^{-1}\{\log ( C_B  \lambda_n K^{-1}(1+\epsilon)\sum_{k} \sqrt{R_{k}} ) - \log C_{\gamma} \}$,
     then with probability at least $1-\epsilon( \max\{I_{\leq k'}, I_{k'<}\})^{-3}$ and for $\lambda_n \geq \|\mathfrak{X}^*(\mE)\|_{\infty}/n$,
    \begin{align}\label{eq:als-bound}
        \|X(\hat{\mG}) - X(\mG^*)\|_F \leq C_A  (1+\epsilon) \lambda_n \sum_{k=1}^{K-1} \sqrt{R_{k}}.
    \end{align}
\end{theorem}

Again, we can obtain the consistency of TT-RALS by setting
$\lambda_n \to 0$ as $n$ increases.  Since the setting of $\lambda_n$
corresponds to that of Theorem \ref{thm:convex}, the speed of
convergence of TT-RALS in terms of $n$ is equivalent to the speed of
TT-ADMM.

By comparing with the convex approach (Theorem~\ref{thm:convex}), the
error rate becomes slightly worse.  Here, the term
$\lambda_n \sum_{k=1}^{K-1} \sqrt{R_{k}}$ in \eqref{eq:als-bound}
comes from the estimation by the alternating minimization, which
linearly increases by $K$.  This is because there are $K$ optimization
problems and their errors are accumulated to the final solution.
The term $(1+\epsilon)$ in \eqref{eq:als-bound} comes from the random
projection.  The size of the error $\epsilon$ can be arbitrary small
by controlling the parameters of the random projection $D_1,D_2$ and
$s$. 
%
%Also, we provide the error bound of TT-RALS without the random
%projection.
%%
%\begin{corollary} \label{cor:rals}
%	Suppose that the settings of Theorem \ref{thm:als} hold.
%	Let $X(\hat{\mG})$ be the solution of TT-RALS without the random projection operator $\mP_k$.
%	Then, with sufficiently large $n,t$ and high probability,
%    \begin{align*}
%        \|X(\hat{\mG}) - X(\mG^*)\|_F \leq C_A  \lambda_n \sum_{k=1}^{K-1} \sqrt{R_{k}}.
%    \end{align*}
%\end{corollary}
%%

%The error bound of Theorem \ref{thm:als} is similar to that of TT-convex in Theorem \ref{thm:convex}.
%Namely, when the initial points of the ALS are close to $\mG^*$ and the number of ALS iteration is sufficiently large, the ALS can obtain the same performance to TT-Convex.

%%% Local Variables:
%%% mode: latex
%%% TeX-master: "TTcomp_NIPS2017.tex"
%%% End:

\section{Related Work}

%We compare the TT-ADMM and TT-RALS with the existing tensor completion methods with TT-decomposition.
To solve the tensor completion problem with TT decomposition,
Wang \etal \cite{wang2016tensor} and Grasedyck \etal \cite{grasedyck2015alternating} developed
algorithms that iteratively solve minimization problems with respect
to $G_k$ for each $k = 1,\ldots,K$. Unfortunately, the adaptivity of
the TT rank is not well discussed.  \cite{wang2016tensor} assumed that
the TT rank is given. Grasedyck \etal \cite{grasedyck2015alternating} proposed a grid
search method. However, the TT rank is determined by a single
parameter (i.e., $R_1=\dots=R_{K-1}$) and the search method lacks its
generality.  Furthermore, the scalability problem remains in both
methods---they require more than $O(I^K)$ space.

Phien et al.(2016) \cite{phien2016efficient} proposed a convex
optimization method using the Schatten TT norm.  However, because they
employed an alternating-type optimization method, the global
convergence of their method is not guaranteed. Moreover, since they
maintain $X$ directly and perform the reshape of $X$ several times,
their method requires $O(I^K)$ time.

%TT-ADMM is the convex optimization method which can be guaranteed to
%converge to the global optimum via the ADMM approach and can select
%the rank adaptivity.  Also, the statistical error is theoretically
%evaluated by the value of TT rank.  Though the advantages, TT-ADMM
%cannot avoid the computational burden from controlling $X$ directly.
%
%TT-RALS is the method which can solve the computational complexity
%problem.  Since TT-RALS does not handle $X$ but controls only
%$\{G_k\}_{k=1}^K$, the computational burden is avoided.  Also, TT-RALS
%can select the TT rank adaptivity and its statistical performance is
%evaluated.  By setting the random projection parameter $s$
%sufficiently small, the time and space complexity does not increase
%exponentially as $K$ grows.

Table~\ref{tab:contribution} highlights the difference between the
existing and our methods. We emphasize that our study is the first
attempt to analyze the statistical performance of TT decomposition. In
addition, TT-RALS is only the method that both time and space
complexities do not grow exponentially in $K$.

\begin{table}[htbp]
  \centering
  {\small
  \begin{tabular}{rccccc}
    \hline
    Method & \shortstack{Global\\Convergence} & \shortstack{Rank\\Adaptivity} & \shortstack{Time\\Complexity}& \shortstack{Space\\Complexity}& \shortstack{Statistical\\Bounds}\\
    \hline
    TCAM-TT\cite{wang2016tensor}&        & & $O(nIKR^4)$ & $O(I^K)$ & \\
    ADF for TT\cite{grasedyck2015alternating}          &          & (search) &$O(KIR^3 + nKR^2)$& $O(I^K)$& \\
    SiLRTC-TT\cite{phien2016efficient}      & & \checkmark & $O(I^{3K/2})$ & $O(KI^K)$ & \\
   	\textbf{TT-ADMM}              &\checkmark & \checkmark & $O(K I^{3K/2})$ & $O(I^K)$ &\checkmark\\
    \textbf{TT-RALS}              &      &  \checkmark  & $O((n + KD^2)KI^2R^4)$ & $O(n + KI^2R^4)$ &\checkmark\\
    \hline
  \end{tabular}
  }
  \caption{Comparison of TT completion algorithms, with $R$ is a parameter for the TT rank such that $R = R_1 = \cdots = R_{K-1}$, $I = I_1 = \cdots = I_K$ is dimension, $K$ is the number of modes, $n$ is the number of observed elements, and $D$ is the dimension of random projection.}
  \label{tab:contribution}
\end{table}

%%% Local Variables:
%%% mode: latex
%%% TeX-master: "TTcomp_NIPS2017.tex"
%%% End:

\section{Experiments}
    \begin{figure}[htbp]
            \begin{center}
                \includegraphics[width=0.95\hsize]{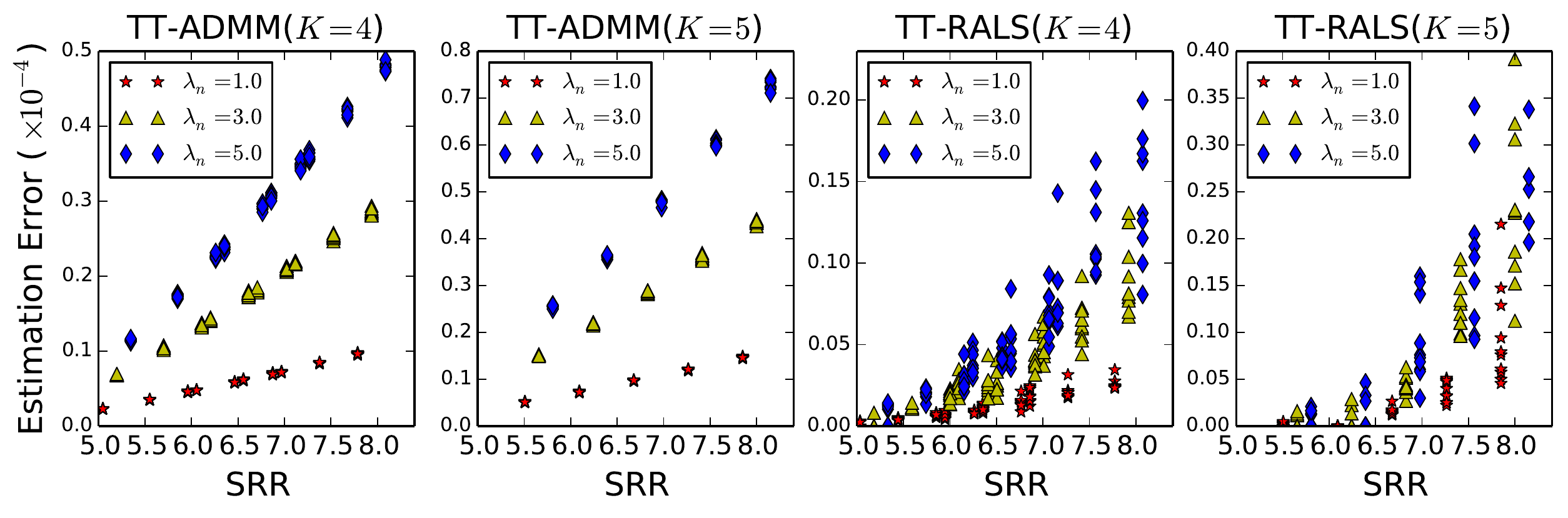}
            \end{center}
            \caption{Synthetic data: the estimation error
              $\|\hat{X} - X^*\|_F$ against SRR $\sum_k \sqrt{R_k}$
              with the order-$4$ tensor $(K=4)$ and the order-$5$
              tensor $(K=5)$. For each rank and $\lambda_n$, we
              measure the error by 10
              trials with different random seeds, which affect both
              the missing pattern and the initial
              points. \label{fig:validation}}
    \end{figure}

        \begin{table}
          \caption{Electricity data: the prediction error and the runtime (in seconds). \label{fig:elec}}
\begin{center}
        	\begin{tabular}{c|cc|cc|cc|cc}
				        	& \multicolumn{2}{|c|}{$K=5$} 		& \multicolumn{2}{|c}{$K=7$} 		& \multicolumn{2}{|c}{$K=8$} 		& \multicolumn{2}{|c}{$K=10$} 		\\ \hline
				Method    & 	Error		& 		Time		& 		Error	& 		Time		& 		Error	& 		Time		& 		Error	& 		Time		\\ \hline
				TCAM-TT &		0.219		&		2.174		&		0.928	&		27.497		&		0.928	&		146.651		&		N/A		&		N/A			\\
				ADF for TT &	0.998		&		1.221		&		1.160	&		23.211		&		1.180	&		278.712		&		N/A		&		N/A			\\
				SiLRTC-TT &		0.339		&		1.478		&		0.928	&		206.708		&		N/A		&		N/A			&		N/A		&		N/A			\\
				TT-ADMM &		0.221		&		0.289		&		1.019	&		154.991		&		1.061	&		2418.00		&		N/A		&		N/A  		\\
				TT-RALS &		0.219		&		4.644		&		0.928	&		4.726		&		0.928	&		7.654		&		1.173	&		7.968				
        	\end{tabular}        
\end{center}
        \end{table}

\subsection{Validation of Statistical Efficiency}

Using synthetic data, we verify the theoretical bounds derived in
Theorems~\ref{thm:convex} and \ref{thm:als}.  We first generate TT
components $\mG^*$; each component $G_k^*$ is generated as
$G_k^* = G_k^{\dagger} / \|G_k^{\dagger}\|_F$ where each element of
$G_k^{\dagger}$ is sampled from the i.i.d.  standard normal
distribution. Then we generate $Y$ by following the generative model
\eqref{model:obs} with the observation ratio $n/\prod_k I_k = 0.5$ and
the noise variance $0.01$.
We prepare two tensors of different size: an order-$4$ tensor of size
$8\times 8\times 10\times 10$ and an order-$5$ tensor of size
$5\times 5\times 7\times 7\times 7$.  At the order-$4$ tensor, the TT
rank is set as $(R_1,R_2,R_3)$ where $R_1,R_2,R_3 \in \{3,5,7\}$.  At
the order-$5$ tensor, the TT rank is set as $(R_1,R_2,R_3,R_4)$ where
$R_1,R_2,R_3,R_4 \in \{2,4\}$.  For estimation, we set the size of
$G_k$ and $\Pi_k$ as $10$, which is larger than the true TT rank.  The
regularization coefficient $\lambda_n$ is selected from $\{1,3,5\}$.
The parameters for random projection are set as $s = 20$ and $D_1=D_2=10$.

Figure \ref{fig:validation} shows the relation between the estimation
error and the sum of root rank (SRR) $\sum_k \sqrt{R_k}$.  The result of TT-ADMM shows that 
the empirical errors are linearly related to SSR which is shown by the theoretical result.
The result of TT-RALS roughly replicates the theoretical relationship.

%These results meet Theorem
%\ref{thm:convex} and Theorem \ref{thm:als}. 

\subsection{Higher-Order Markov Chain for Electricity Data}

We apply the proposed tensor completion methods for analyzing the
electricity consumption data~\cite{Lichman:2013}.  The dataset
contains time series measurements of household electric power
consumption for every minutes from December 2006 to November 2010 and
it contains over $200,000$ observations.

The higher-order Markov chain is a suitable method to represent
long-term dependency, and it is a common tool of time-series
analysis~\cite{hamilton1994time} and natural language
processing~\cite{jurafsky2014speech}.  Let $\{W_t\}_{t}$ be
discrete-time random variables take values in a finite set $B$, and
the order-$K$ Markov chain describes the conditional distribution of
$W_t$ with given $\{W_{\tau}\}_{\tau < t}$ as
$P(W_t | \{W_{\tau}\}_{\tau < t}) = P(W_t | W_{t-1}, \ldots,
W_{t-K})$.
As $K$ increases, the conditional distribution of $W_t$ can include
more information from the past observations.

%For practical use, the higher-order Markov chain has two problems: the memory requirement and unobserved path.
%Since it is necessary to storage $|B|^{M+1}$ numerics to represent the $M$-th Markov chain, there exists computational burden to compute the transition.
%Also, the variation of the realized path increases with large $M$, a part of the path is not observed unless the length of the observation is sufficiently long.

We complete the missing values of $K$-th Markov transition of the
electricity dataset.  We discretize the value of the dataset into $10$
values and set $K \in \{ 6,7,8,10\}$.  Next, we empirically estimate
the conditional distribution of size $10^K$ using $200,000$
observations.  Then, we create $X$ by randomly selecting $n=10,000$
elements from the the conditional distribution and regarding the other
elements as missing. After completion by the TT methods, the
prediction error is measured.  We select hyperparameters using a grid
search with cross-validation.

Figure~\ref{fig:elec} compares the prediction error and the runtime.
When $K=5$, the rank adaptive methods achieve low estimation errors.
As $K$ increases, however, all the methods except for TT-RALS suffers
from the scalability issue. Indeed, at $K=10$, only TT-RALS works and
the others does not due to exhausting memory.
%For accuracy, TT-RALS outperforms the others because it selects the TT rank appropriately.

%%% Local Variables:
%%% mode: latex
%%% TeX-master: "TTcomp_NIPS2017.tex"
%%% End:

\section{Conclusion}

In this paper, we investigated TT decomposition from the statistical
and computational viewpoints. To analyze its statistical performance,
we formulated the convex tensor completion problem via the low-rank TT
decomposition using the TT Schatten norm. In addition, because the
optimization of the convex problem is infeasible, we developed an
alternative algorithm called TT-RALS by combining with the ideas of
random projection and alternating minimization. Based on this, we
derived the error bounds of estimation for both the convex minimizer
and the solution obtained by TT-RALS. The experiments supported our
theoretical results and demonstrated the scalability of TT-RALS.

%%% Local Variables:
%%% mode: latex
%%% TeX-master: "TTcomp_NIPS2017.tex"
%%% End:

%\section*{References}
%\setlength{\bibsep}{0pt plus 0.3ex}

\newpage

\appendix

\if0
\begin{center}
	\begin{Large}
		Supplementary Material for \\
		``On Tensor Train Rank Minimization : Statistical Efficiency and Scalable Algorithm''
	\end{Large}
\end{center}
\fi

\section{Proof of Theorem \ref{thm:random}}

The theorem is obtained immediately by combining Li et al.~\cite{li2006very} and Mu et al.~\cite{mu2011accelerated}.

Let $\Pi: \mathbb{R}^n \to \mathbb{R}^k$ be a sparse random projection defined by 
\begin{align}
  \Pi_{ij} = \begin{cases} +\sqrt{s/k} & \text{probability } 1/2s, \\ 0 & \text{probability } 1 - 1/s, \\ -\sqrt{s/k} & \text{probability } 1/2s. \end{cases}
\end{align}
Then the following theorem holds.
\begin{lemma}[Lemma 4 in \cite{li2006very}]
\label{lem:sparseJL}
Let $u$ be a unit vector. Then
$\sqrt{k} (\Pi u)_i \to N(0,1)$ and $k \| \Pi u \|_2^2 \to \chi_k^2$ in law with the convergence rate 
\begin{align}
  |P(\sqrt{k} (\Pi u)_i < t) - P(N(0,1) < t)| \le 0.8 \sqrt{s} \sum_j |u_j|^3.
\end{align}
%If $s = \Omega(\alpha^2 \log (n / \epsilon \delta))$ and $k = \Omega(\log (1/\delta)/\epsilon^2)$, for every vector $x \in \mathbb{R}^n$ satisfying $\| x \|_\infty \le \alpha \| x \|_2$, we have
%\begin{align}
%  (1 - \epsilon) \| x \|_2 \le \| \Pi x \|_2 \le (1 + \epsilon) \| x \|_2.
%\end{align}
%with probability at least $1 - O(\delta)$.
\qed
\end{lemma}
Suppose that $\sum_j |u_j|^3 \le e^{-k \epsilon^2/4} / (1.6 k \sqrt{s})$. 
Then 
\begin{align}
  |P(\|\Pi u\|_2^2 \in [1-\epsilon, 1+\epsilon]) - P(\chi_k^2  \in [k(1-\epsilon), k(1+\epsilon)])| \le e^{-k \epsilon^2/4}.
\end{align}
By using 
\begin{align}
  P(\chi_k^2  \in [k(1-\epsilon), k(1+\epsilon)]) \le 2 e^{-k (\epsilon^2 - \epsilon^3)/4},
\end{align}
we have
\begin{align}
  P(\| \Pi u \|_2^2 \in [1-\epsilon, 1+\epsilon]) \le 3 e^{-k (\epsilon^2 - \epsilon^3)/4}.
\end{align}

The preservation of $L_2$ norm implies the preservation of the Schatten-$1$ norm as follows.
\begin{lemma}[Restatement of Theorem~1 in \cite{mu2011accelerated}]
\label{lem:schattenbound}
Let $Z$ be an $m \times n$ matrix with rank $r$.
If $k \ge r$ and $\| \Pi u \|_2^2 \in [1-\epsilon, 1+\epsilon]$ for all singular vectors $u$ of $Z$, we have
\begin{align}
  \sqrt{(1 - \epsilon)/r} \| Z \|_s \le \| \Pi Z \|_s \le \sqrt{1 + \epsilon} \| Z \|_s.
\end{align}
\qed
\end{lemma}

Now we prove Theorem~\ref{thm:random}.
Let $Z = Q_k(X)$ be a $\prod_{k'=1}^k I_{k'} \times \prod_{k'=k+1}^K I_{k'}$ matrix obtained by reshaping tensor $X$.
Since $X$ is a TT of rank $(R_1, \ldots, R_K)$, the rank of matrix $Z$ is at most $R_k$.
By applying Lemma~\ref{lem:schattenbound} twice, we obtain
\begin{align}
  \frac{1 - \epsilon}{R_k} \| Z \|_s \le \| \Pi Z \|_s \le (1 + \epsilon) \| Z \|_s
\end{align}
with some probability. 
If $\sum_j |u_j|^3 \le e^{-k \epsilon^2/4} / (1.6 k \sqrt{s})$ for all singular vectors of $Z$, 
the probability is at least $1 - 6 R_k e^{-k \epsilon^2/4}$ since there are $2 R_k$ singular vectors.

\section{Proof of Theorem \ref{thm:convex}}

\begin{proof}

Since $\hat{X}$ is the minimizer of the optimization problem, we have the following basic inequality
\begin{align*}
	\frac{1}{2n}\|Y - \mathfrak{X}(\hat{X})\|^2 + \lambda_n \sum_{k=1}^{K-1} \|Q_k(\hat{X})\|_{s} \leq \frac{1}{2n}\|Y - \mathfrak{X}(\hat{X}^*)\|^2 + \lambda_n  \sum_{k=1}^{K-1} \|Q_k(X^*)\|_{s}.
\end{align*}
Using the relation that
\begin{align*}
	\|Y - \mathfrak{X}(\hat{X})\|^2& = \|(Y - \mathfrak{X}(X^*)) -( \mathfrak{X}(\hat{X}) - \mathfrak{X}(X^*))\|^2 \\
	&= \|Y - \mathfrak{X}(X^*)\|^2 + \| \mathfrak{X}(\hat{X}) - \mathfrak{X}(X^*)\|^2 - 2\langle Y - \mathfrak{X}(X^*), \mathfrak{X}(\hat{X}) - \mathfrak{X}(X^*)\rangle,
\end{align*}
we rewrite the basic inequality as
\begin{align*}
	\frac{1}{2n}\| \mathfrak{X}(\hat{X}) - \mathfrak{X}(X^*)\|^2 \leq \frac{1}{n} \langle \mathfrak{X}(\hat{X}) - \mathfrak{X}(X^*), \mE \rangle + \lambda_n \sum_{k=1}^{K-1} \left( \|Q_k(X^*)\|_s - \|Q_k(\hat{X})\|_s \right).
\end{align*}
Define the error $\Delta := \hat{X} - X^*$.
Applying the H\"older's inequality, we have
\begin{align*}
	&\frac{1}{n}\langle \mathfrak{X}(\Delta) , \mE \rangle = \frac{1}{n} \langle \Delta , \mathfrak{X}^*(\mE) \rangle = \frac{1}{n}\frac{1}{K-1}\sum_{k=1}^{K-1} \langle Q_k(\Delta), \mathfrak{X}^*(\mE) \rangle \\
	&\leq \frac{1}{n} \frac{1}{K-1}\sum_{k=1}^{K-1} \|\mE\|_{\infty} \|Q_k(\Delta)\|_{s} \leq \frac{\lambda_n }{K-1} \sum_{k=1}^{K-1}\|Q_k(\Delta)\|_{s} ,
\end{align*}
where $\mathfrak{X}^*$ is an adjoint operator of $\mathfrak{X}$ and the last inequality holds by the setting of $\lambda_n$.
Also, the triangle inequality and the linearity of $Q_k(\cdot)$ yield
\begin{align*}
	\|Q_k(X^*)\|_s - \|Q_k(\hat{X})\|_s \leq \|Q_k(\Delta)\|_s.
\end{align*}
Then, we bound the inequality as
\begin{align*}
	\frac{1}{2n}\|\mathfrak{X}(\Delta)\|^2 &\leq \frac{\lambda_n }{K-1} \sum_{k=1}^{K-1}\|Q_k(\Delta)\|_{s} + \frac{ \lambda_n}{K-1} \sum_{k=1}^{K-1}\|Q_k(\Delta)\|_s =  \frac{2\lambda_n}{K-1}\sum_{k=1}^{K-1}\|Q_k(\Delta)\|_{s} .
\end{align*}
To bound $\|Q_k(\Delta)\|_{s}$,  we apply the result of Lemma 1 in \cite{negahban2011estimation} and Lemma 2 in \cite{tomioka2011statistical}.
Along with proof of the lemmas, we obtain the property that a rank of $Q_k(\Delta)$ is bounded by $2R_k$, thus the Cauchy-Schwartz inequality implies
\begin{align*}
	\|Q_k(\Delta)\|_s \leq  \sqrt{2R_k }\|Q_k(\Delta)\|_F.
\end{align*}
Then we obtain
\begin{align*}
	\|\mathfrak{X}(\Delta)\|_F^2 &\leq  \frac{2 \lambda_n}{K-1}\sum_{k=1}^{K-1} \sqrt{2R_k }\|Q_k(\Delta)\|_F.
\end{align*}

We apply the completion theory by \cite{candes2010matrix, candes2012exact} to bound $\|\mathfrak{X}(\Delta)\|_F^2$ below.
Let $k' \in \{1,\ldots,K\}$ be the index which satisfies Assumption \ref{asmp:incoherence_convex}, and we have $\|\mathfrak{X}(\Delta)\|^2 = \|\tilde{\mathfrak{X}}(Q_{k'}(\Delta))\|^2$ where $\tilde{\mathfrak{X}}$ is a rearranging operator for the reshaped tensor.
Then, Theorem 7 in \cite{candes2010matrix} yields that 
\begin{align*}
	\|Q_{k'}(\Delta)\|_F \leq \left(  \sqrt{\frac{48  \min\{I_{\leq k'}, I_{k'<}\}}{n}}+ 1 \right) \|\tilde{\mathfrak{X}}(Q_{k'}(\Delta))\|,
\end{align*}
with probability at least $1-( \max\{I_{\leq k'}, I_{k'<}\})^{-3}$ and
\begin{align*}
	n \geq C_{m'} \mu_{k'}^2  \max\{I_{\leq k'}, I_{k'<}\} R_{k'} \log^3  \max\{I_{\leq k'}, I_{k'<}\},
\end{align*}
with a constant $C_{m'} > 0$.
%Here, we define $\tilde{I}_{k'}^+ := \max \{ \prod_{\ell \leq k'} I_{\ell}, \prod_{\ell > k'} I_{\ell}\}$ and  $\tilde{I}_{k'}^- := \max \{ \prod_{\ell \leq k'} I_{\ell}, \prod_{\ell > k'} I_{\ell}\}$.
Then we obtain that 
\begin{align*}
	\frac{1}{n}\|\tilde{\mathfrak{X}}(Q_{k'}({\Delta}))\|^2 \geq   C_{\kappa'} \|Q_{k'}(\Delta)\|_F^2 = C_{\kappa'} \|\Delta\|_F^2,
\end{align*}
where $C_{\kappa} = (144  \min\{I_{\leq k'}, I_{k'<}\} + 3n)^{-1} > 0$.

Finally, we have
\begin{align*}
	\|\Delta\|_F^2 &\leq  C_{\kappa '}^{-1 } \frac{2 \lambda_n}{K-1}\sum_{k=1}^{K-1} \sqrt{2R_k }\|Q_k(\Delta)\|_F \\
	&= C_{\kappa '}^{-1 }\|\Delta\|_F \frac{2 \lambda_n}{K-1}\sum_{k=1}^{K-1} \sqrt{2R_k }\\
	&= 3C_{\kappa '}^{-1 }\|\Delta\|_F \frac{\lambda_n}{K}\sum_{k=1}^{K-1} \sqrt{2R_k }.
\end{align*}
Dividing both hands side by $\|\Delta\|_F$ provides the result.

\end{proof}

\section{Proof of Theorem \ref{thm:als}}

Preliminarily, we introduce an alternative formation of the optimization problem.
For each $k \in \{1,2,\ldots,K\}$, we rewrite the term $X_k (\mG)$ as
\begin{align*}
	G_k \times_2 G_{<k} \times_3 G_{k<},
\end{align*}
where $\times_j$ denotes the $j$-mode product (for detail, see \cite{kolda2009tensor}).
Here, $G_{<k}$ is a tensor with size $R_k \times I_1 \times \cdots \times I_{K-1}$ and its element is given as
\begin{align*}
	[G_{<k}]_{r,j_1,\ldots,j_{k-1}} = [G_1]_{j_1,:,:}[G_2]_{j_2,:,:} \cdots [G_{k-1}]_{j_{k-1},:,r},
\end{align*}
for $j_{k'} = 1,\ldots,I_{k'}$ and $r = 1,\ldots,R_k$.
Namely, $G_{<k}$ is the left side of tensor train decomposition of $X$ than $G_k$.
Similarly, $G_{k<}$ is a tensor with size $R_{k+1} \times I_{k+1} \times \cdots \times I_{K}$ and its element is given as
\begin{align*}
	[G_{k<}]_{r,j_{k+1},\ldots,j_{K}} = [G_{k+1}]_{j_{k+1},r,:}\cdots [G_{K}]_{j_{K},:,:},
\end{align*}
for $j_{k'} = 1,\ldots,I_{k'}$ and $r = 1,\ldots,R_{k+1}$.
Using the result, the ALS optimization problem \eqref{opt:als2} is rewritten as
\begin{align}
%	\min_{G_k} \left[ \frac{1}{2n} \|Y - \mathfrak{X}(G_k \times_2 G_{<k} \times_3 G_{k<})\|^2 + \frac{\lambda_n}{K-1}\sum_{k'=1}^{K-1} \|Q_{k'}(G_k \times_2 G_{<k} \times_3 G_{k<})\|_{s} \right]. \label{opt:als3}
	\min_{G_k} \left[ \frac{1}{2n} \|Y - \mathfrak{X}(G_k \times_2 G_{<k} \times_3 G_{k<})\|^2 + \frac{\lambda_n}{K-1}\sum_{k'=1}^{K-1} \|\mP_{k'}(G_k \times_2 G_{<k} \times_3 G_{k<})\|_{s} \right]. \label{opt:als3}
\end{align}
When $k=1$, we set $G_{<k} = 1$.
Similarly, when $k=K$, $G_{k<} = 1$ holds.

Using the formula, we investigate the convergence of $G_k$ by fixing other elements as $G_{<k} = \tilde{G}_{<k}$ and $G_{k<} = \tilde{G}_{k<}$.
Let $\{G_k^*\}_{k^*=1}^K$ be a set of tensor which formulates the true tensor $X^*$.
Also, $G_{<k}^*$ and $G_{k<}^*$ are defined similarly.
To evaluate the convergence, we introduce that
\begin{align*}
	\Xi(\tilde{\mG}) := \max_{k \in \{1,\ldots,K\}} \left[\|\tilde{G}_k - G_k^*\|_F \right].
\end{align*}

We obtain the following lemma which evaluates the optimization of \eqref{opt:als3} with given $\tilde{G}_{<k}$ and $\tilde{G}_{k<}$.
\begin{lemma} \label{lem:als_1}
	For each $k \in \{1,\ldots,K\}$, consider the optimization \eqref{opt:als3} with respect to $G_k$ with given $\tilde{\mG}$.
	Then, with probability at least $1-( \max\{I_{\leq k'}, I_{k'<}\})^{-3}$ and 
	\begin{align*}
		n \geq C_{m} \mu_{k'}^2  \max\{I_{\leq k'}, I_{k'<}\} R_{k'} \log^n  \max\{I_{\leq k'}, I_{k'<}\},
	\end{align*}
	we obtain
	\begin{align*}
			\|\hat{G}_k - G_k^*\| \leq 6 (C_{\kappa} C_K)^{-1} \left\{ 2(K-1)C_K(1+n^{-1})  \Xi(\tilde{\mG})  +   \frac{2 \lambda_n (2 + \epsilon) }{K-1}\sum_{k'=1}^{K-1} \sqrt{2R_{k'}}\right\}.
	\end{align*}
\end{lemma}
\begin{proof}

Our proof takes following four steps: 1) derive a basic inequality from the optimality condition, 2) bound terms of the RHS of the basis inequality, 3) bound below the LHS of the basis inequality, and 4) combine the result.

\paragraph{Step 1.} Derive a basic inequality.

By the optimality condition of \eqref{opt:als3} with given $\tilde{G}_{<k}$ and $\tilde{G}_{k<}$, we have
\begin{align}
	 &\frac{1}{2n} \|Y - \mathfrak{X}(\hat{G}_k \times_2 \tilde{G}_{<k} \times_3 \tilde{G}_{k<})\|^2 + \frac{\lambda_n}{K-1}\sum_{k'=1}^{K-1} \|\mP_{k'}(\hat{G}_k \times_2 \tilde{G}_{<k} \times_3 \tilde{G}_{k<})\|_{s}  \notag \\
	 & \quad \leq \frac{1}{2n} \|Y - \mathfrak{X}(G_k^* \times_2 \tilde{G}_{<k} \times_3 \tilde{G}_{k<})\|^2 + \frac{\lambda_n}{K-1}\sum_{k'=1}^{K-1} \|\mP_{k'}(G_k^* \times_2 \tilde{G}_{<k} \times_3 \tilde{G}_{k<})\|_{s}. \label{cond:opt}
\end{align}
Using the triangle inequality and the linearity of $\mathfrak{X}$ and the mode product $\times_j$, we obtain 
\begin{align*}
	&\|Y - \mathfrak{X}(\hat{G}_k \times_2 \tilde{G}_{<k} \times_3 \tilde{G}_{k<})\|^2 \\
	&=\|\{Y - \mathfrak{X}(G_k^* \times_2 \tilde{G}_{<k} \times_3 \tilde{G}_{k<})\} -\{\mathfrak{X}(\hat{G}_k \times_2 \tilde{G}_{<k} \times_3 \tilde{G}_{k<}) - \mathfrak{X}(G_k^* \times_2 \tilde{G}_{<k} \times_3 \tilde{G}_{k<})\} \|^2 \\
	&= \|\{Y - \mathfrak{X}(G_k^* \times_2 \tilde{G}_{<k} \times_3 \tilde{G}_{k<})\} -\mathfrak{X}((\hat{G}_k - G_k^*) \times_2 \tilde{G}_{<k} \times_3 \tilde{G}_{k<})  \|^2 \\
	& = \|Y - \mathfrak{X}(G_k^* \times_2 \tilde{G}_{<k} \times_3 \tilde{G}_{k<})\|^2 + \| \mathfrak{X}((\hat{G}_k - G_k^*) \times_2 \tilde{G}_{<k} \times_3 \tilde{G}_{k<})  \|^2  \\
	& \quad -2 \langle Y - \mathfrak{X}(G_k^* \times_2 \tilde{G}_{<k} \times_3 \tilde{G}_{k<}), \mathfrak{X}((\hat{G}_k - G_k^*) \times_2 \tilde{G}_{<k} \times_3 \tilde{G}_{k<}) \rangle.
\end{align*}
Substituting the result into \eqref{cond:opt}, we obtain that 
\begin{align}
	 &\frac{1}{2n} \| \mathfrak{X}((\hat{G}_k - G_k^*) \times_2 \tilde{G}_{<k} \times_3 \tilde{G}_{k<})  \|^2 + \frac{\lambda_n}{K-1}\sum_{k'=1}^{K-1} \|\mP_{k'}(\hat{G}_k \times_2 \tilde{G}_{<k} \times_3 \tilde{G}_{k<})\|_{s}  \notag \\
	 & \quad \leq \frac{1}{n} \langle Y - \mathfrak{X}(G_k^* \times_2 \tilde{G}_{<k} \times_3 \tilde{G}_{k<}), \mathfrak{X}((\hat{G}_k - G_k^*) \times_2 \tilde{G}_{<k} \times_3 \tilde{G}_{k<}) \rangle \notag \\
	 & \quad \quad  + \frac{\lambda_n}{K-1}\sum_{k'=1}^{K-1} \|\mP_{k'}(G_k^* \times_2 \tilde{G}_{<k} \times_3 \tilde{G}_{k<})\|_{s}. \label{ineq:basic0}
\end{align}
About the regularization term, we apply the following inequality
\begin{align}
	&\frac{\lambda_n}{K-1}\sum_{k'=1}^{K-1} \|\mP_{k'}(G_k^* \times_2 \tilde{G}_{<k} \times_3 \tilde{G}_{k<})\|_{s} - \frac{\lambda_n}{K-1}\sum_{k'=1}^{K-1} \|\mP_{k'}(\hat{G}_k \times_2 \tilde{G}_{<k} \times_3 \tilde{G}_{k<})\|_{s} \notag \\
	&\leq \frac{\lambda_n}{K-1}\sum_{k'=1}^{K-1} \|\mP_{k'}((G_k^* - \hat{G}_k) \times_2 \tilde{G}_{<k} \times_3 \tilde{G}_{k<})\|_{s} \notag \\
	&= \frac{\lambda_n}{K-1}\sum_{k'=1}^{K-1} \left(\|Q_{k'}((G_k^* - \hat{G}_k) \times_2 \tilde{G}_{<k} \times_3 \tilde{G}_{k<})\|_{s}  \right. \notag \\
	& \quad  \left. -  \|\mP_{k'}((G_k^* - \hat{G}_k) \times_2 \tilde{G}_{<k} \times_3 \tilde{G}_{k<})\|_{s} -  \|Q_{k'}((G_k^* - \hat{G}_k) \times_2 \tilde{G}_{<k} \times_3 \tilde{G}_{k<})\|_{s}  \right), \label{ineq:pq}
\end{align}
by using the triangle inequality and the linearity of the random projection operator $\mP_{k'}$.
Here, we apply Theorem \ref{thm:random} with $\epsilon$ and obtain
\begin{align*}
	 &\left( \frac{1 - \epsilon}{2 R_{k'}} - 1 \right)\|Q_{k'}((G_k^* - \hat{G}_k) \times_2 \tilde{G}_{<k} \times_3 \tilde{G}_{k<})\|_{s} \\
	 &\quad \leq \|\mP_{k'}((G_k^* - \hat{G}_k) \times_2 \tilde{G}_{<k} \times_3 \tilde{G}_{k<})\|_{s} -  \|Q_{k'}((G_k^* - \hat{G}_k) \times_2 \tilde{G}_{<k} \times_3 \tilde{G}_{k<})\|_{s} \\
	 &\quad \leq \epsilon \|Q_{k'}((G_k^* - \hat{G}_k) \times_2 \tilde{G}_{<k} \times_3 \tilde{G}_{k<})\|_{s}.
\end{align*}
Here, the denominator in the left hand side follows Lemma 1 in \cite{negahban2011estimation}.
Then we have
\begin{align*}
	 &\left| \|\mP_{k'}((G_k^* - \hat{G}_k) \times_2 \tilde{G}_{<k} \times_3 \tilde{G}_{k<})\|_{s} -  \|Q_{k'}((G_k^* - \hat{G}_k) \times_2 \tilde{G}_{<k} \times_3 \tilde{G}_{k<})\|_{s}  \right| \\
	 &\leq \max\{\epsilon,| (1-\epsilon)/(2R_{k'}) - 1| \}\|Q_{k'}((G_k^* - \hat{G}_k) \times_2 \tilde{G}_{<k} \times_3 \tilde{G}_{k<})\|_{s}\\
	 & \leq (1 + \epsilon) \|Q_{k'}((G_k^* - \hat{G}_k) \times_2 \tilde{G}_{<k} \times_3 \tilde{G}_{k<})\|_{s}.
\end{align*}
Using this result and continue \eqref{ineq:pq} then we have
\begin{align}
	\eqref{ineq:pq} \leq \frac{\lambda_n}{K-1}\sum_{k'=1}^{K-1} (2+\epsilon)\|Q_{k'}((G_k^* - \hat{G}_k) \times_2 \tilde{G}_{<k} \times_3 \tilde{G}_{k<})\|_{s}. \label{ineq:pq2}
\end{align}

About the first term of the RHS of \eqref{ineq:basic0}, we decompose it as
\begin{align*}
	&\frac{1}{n} \langle Y - \mathfrak{X}(G_k^* \times_2 \tilde{G}_{<k} \times_3 \tilde{G}_{k<}), \mathfrak{X}((\hat{G}_k - G_k^*) \times_2 \tilde{G}_{<k} \times_3 \tilde{G}_{k<}) \rangle \\
	&= \frac{1}{n} \langle Y - \mathfrak{X}(G_k^* \times_2 G_{<k}^* \times_3 G_{k<}^*), \mathfrak{X}((\hat{G}_k - G_k^*) \times_2 \tilde{G}_{<k} \times_3 \tilde{G}_{k<}) \rangle \\
	&\quad  + \frac{1}{n} \langle  \mathfrak{X}(G_k^* \times_2 (\tilde{G}_{<k} - G_{<k}^*)\times_3 G_{k<}^*), \mathfrak{X}((\hat{G}_k - G_k^*) \times_2 \tilde{G}_{<k} \times_3 \tilde{G}_{k<}) \rangle \\
	&\quad  + \frac{1}{n} \langle  \mathfrak{X}(G_k^* \times_2 G_{<k}^*\times_3 (\tilde{G}_{k<} - G_{k<}^*)), \mathfrak{X}((\hat{G}_k - G_k^*) \times_2 \tilde{G}_{<k} \times_3 \tilde{G}_{k<}) \rangle \\
	&\quad  + \frac{1}{n} \langle  \mathfrak{X}(G_k^* \times_2 (\tilde{G}_{<k} - G_{<k}^*)\times_3 (\tilde{G}_{k<} - G_{k<}^*)), \mathfrak{X}((\hat{G}_k - G_k^*) \times_2 \tilde{G}_{<k} \times_3 \tilde{G}_{k<}) \rangle \\
	&=: T_0 + T_1 + T_2 + T_3.
\end{align*}
About the term $T_0$, we use the observation model \eqref{model:obs} and the adjoint operator $\mathfrak{X}^*$ then obtain
\begin{align*}
	T_0 &= \frac{1}{n} \langle \mE, \mathfrak{X}((\hat{G}_k - G_k^*) \times_2 \tilde{G}_{<k} \times_3 \tilde{G}_{k<} )\rangle \\
	&= \frac{1}{n} \langle  \mathfrak{X}^*(\mE), (\hat{G}_k - G_k^*) \times_2 \tilde{G}_{<k} \times_3 \tilde{G}_{k<} \rangle.
\end{align*}
Since the reshaping does not affect the value of the inner product, we continue to evaluate $T_0$ as
\begin{align*}
	T_0 &=\frac{1}{n} \langle  \mathfrak{X}^*(\mE), (\hat{G}_k - G_k^*) \times_2 \tilde{G}_{<k} \times_3 \tilde{G}_{k<} \rangle \\
	& =\frac{1}{n(K-1)} \sum_{k'=1}^{K-1} \langle Q_{k'}(\mathfrak{X}^*(\mE)), Q_{k'} ( (\hat{G}_k - G_k^*) \times_2 \tilde{G}_{<k} \times_3 \tilde{G}_{k<}) \rangle \\
	& \leq \frac{1}{n(K-1)} \sum_{k'=1}^{K-1} \|Q_{k'}(\mathfrak{X}^*(\mE))\|_{\infty} \| Q_{k'} ( (\hat{G}_k - G_k^*) \times_2 \tilde{G}_{<k} \times_3 \tilde{G}_{k<})\|_{s} \\
	&= \frac{1}{n(K-1)}\|\mathfrak{X}^*(\mE)\|_{\infty} \sum_{k'=1}^{K-1}  \| Q_{k'} ( (\hat{G}_k - G_k^*) \times_2 \tilde{G}_{<k} \times_3 \tilde{G}_{k<})\|_{s} \\
	& \leq \frac{\lambda_n}{(K-1)} \sum_{k'=1}^{K-1}  \| Q_{k'} ( (\hat{G}_k - G_k^*) \times_2 \tilde{G}_{<k} \times_3 \tilde{G}_{k<})\|_{s}.
\end{align*}
The first inequality follows the H\"older's inequality, and the second inequality is derived by the setting of $\lambda_n$.

Substituting \eqref{ineq:pq2} and the bounds with $T_1,T_2,T_3$ and $T_0$ into \eqref{ineq:basic0}, finally we obtain
\begin{align}
	 &\frac{1}{2n} \| \mathfrak{X}((\hat{G}_k - G_k^*) \times_2 \tilde{G}_{<k} \times_3 \tilde{G}_{k<})  \|^2  \notag \\
	 & \leq T_1 + T_2 + T_3 + \underbrace{\frac{\lambda_n }{K-1} \sum_{k'=1}^{K-1} (2+\epsilon) \|Q_{k'}((G_k^* - \hat{G}_k) \times_2 \tilde{G}_{<k} \times_3 \tilde{G}_{k<})\|_{s}}_{=:T_4}. \label{ineq:basic1}
\end{align}
Here, we obtain the basic inequality.

\paragraph{Step 2.} Bound the RHS of the basic inequality.

For brevity, we introduce notation 
\begin{align*}
	\tilde{\Delta}_k := (\hat{G}_k - G_k^*) \times_2 \tilde{G}_{<k} \times_3 \tilde{G}_{k<}.
\end{align*}

We bound $T_1$ by using the Cauchy-Schwartz inequality as
\begin{align*}
T_1 &= \frac{1}{n} \langle  \mathfrak{X}(G_k^* \times_2 (\tilde{G}_{<k} - G_{<k}^*)\times_3 G_{k<}^*), \mathfrak{X}((\hat{G}_k - G_k^*) \times_2 \tilde{G}_{<k} \times_3 \tilde{G}_{k<}) \rangle \\
	& \leq \frac{1}{n} \|\mathfrak{X}(G_k^* \times_2 (\tilde{G}_{<k} - G_{<k}^*)\times_3 G_{k<}^*)\| \|\mathfrak{X}((\hat{G}_k - G_k^*) \times_2 \tilde{G}_{<k} \times_3 \tilde{G}_{k<})\| \\
	& \leq \frac{1}{n} \|G_k^* \times_2 (\tilde{G}_{<k} - G_{<k}^*)\times_3 G_{k<}^*\|_F \|(\hat{G}_k - G_k^*) \times_2 \tilde{G}_{<k} \times_3 \tilde{G}_{k<}\|_F\\
	& \leq \frac{1}{n} \|G_k^* \times_2 (\tilde{G}_{<k} - G_{<k}^*)\times_3 G_{k<}^*\|_F \|\tilde{\Delta}_k\|_F,	
\end{align*}
here we use the relation $\|\mathfrak{X}(X)\|^2 \leq \|X\|_F^2$ for all $X \in \Theta$.
We introduce a constant $c_k$ for $k = 1,\ldots,K$ which satisfying $c_k \geq \| A \times_k G_k^* \|_F / \|A\|_F$ where $A$ is a tensor with proper size.
Since we suppose that the reshaped matrix from $G_k^*$ has $R_k$ row rank, we can guarantee that $c_k$ is positive and finite.
Using $c_k$, we have
\begin{align*}
	&\|G_k^* \times_2 (\tilde{G}_{<k} - G_{<k}^*)\times_3 G_{k<}^*\|\\
	&\leq c_k \prod_{k' > k}c_{k'} \|\tilde{G}_{<k} - G_{<k}^*\|_F \\
	& \leq c_k \prod_{k' > k}c_{k'} \sum_{k' < k} \|\tilde{G}_{k'} - G_{k'}^*\|_F \prod_{\ell < k, \ell \neq k'} c_{\ell} \\
	& \leq \prod_{k' \geq k}c_{k'} (k-1) \prod_{\ell < k}c_{\ell} \Xi(\tilde{\mG}) \\
	&= (k-1) \prod_{k'=1}^K c_{k'} \Xi(\tilde{\mG}).
\end{align*}
Here, we define $C_K :=  \prod_{k'=1}^K c_{k'}$, we obtain
\begin{align}
	T_1 \leq \frac{1}{n}(k-1)C_K\Xi(\tilde{\mG}). \label{ineq:t1}
\end{align}

Similarly, we obtain
\begin{align}
	T_2 \leq \frac{1}{n}(K-k)C_K\Xi(\tilde{\mG}). \label{ineq:t2}
\end{align}

About $T_3$, we have
\begin{align*}
	T_3 &=  \frac{1}{n} \langle  \mathfrak{X}(G_k^* \times_2 (\tilde{G}_{<k} - G_{<k}^*)\times_3 (\tilde{G}_{k<} - G_{k<}^*)), \mathfrak{X}((\hat{G}_k - G_k^*) \times_2 \tilde{G}_{<k} \times_3 \tilde{G}_{k<}) \rangle \\
	&\leq  \frac{1}{n} \|G_k^* \times_2 (\tilde{G}_{<k} - G_{<k}^*)\times_3 (\tilde{G}_{k<} - G_{k<}^*)\|_F \|\Delta_k\|_F.
\end{align*}
We evaluate the first norm as
\begin{align*}
	&\|G_k^* \times_2 (\tilde{G}_{<k} - G_{<k}^*)\times_3 (\tilde{G}_{k<} - G_{k<}^*)\|_F \\
	&\leq c_k \left( \|\tilde{G}_{<k}\times_3 (\tilde{G}_{k<} - G_{k<}^*)\|_F + \|{G}_{<k}^*\times_3 (\tilde{G}_{k<} - G_{k<}^*)\|_F \right) \\
	& \leq \frac{2}{n} (K-1) C_K \Xi(\tilde{\mG}).
\end{align*}
Then, we have
\begin{align}
	T_3 \leq \frac{2}{n}(K-1)C_K\Xi(\tilde{\mG}). \label{ineq:t3}
\end{align}

To bound $T_4$, we apply the same line of the proof of Theorem \ref{thm:convex}.
Along with Lemma 1 in \cite{negahban2011estimation}, we bound the Schatten norm of $Q_k(\tilde{\Delta}_k)$ and apply the Cauchy-Schwartz inequality, then obtain
\begin{align*}
	T_4 \leq \frac{2 \lambda_n (2+\epsilon)}{K-1} \sum_{k'=1}^{K-1} \sqrt{2R_{k'}}  \|Q_{k'}(\tilde{\Delta}_{k})\| =  \frac{2 \lambda_n (2+\epsilon)}{K-1}\sum_{k'=1}^{K-1} \sqrt{2R_{k'}}  \|\tilde{\Delta}_{k}\|.
\end{align*}

Combining the bound with \eqref{ineq:t1}, \eqref{ineq:t2} and \eqref{ineq:t3}, we update the bound \eqref{ineq:basic1} as
\begin{align}
	\frac{1}{2n} \| \mathfrak{X}(\tilde{\Delta}_k)\|^2 &\leq \frac{3(K-1)C_K}{n}\Xi(\tilde{\mG}) \|\mathfrak{X}(\tilde{\Delta}_k)\| + \frac{2 \lambda_n (2+\epsilon)}{K-1}\sum_{k'=1}^{K-1} \sqrt{2R_{k'}}  \|\tilde{\Delta}_k\| \notag \\
	& \leq \frac{3(K-1)C_K}{n}\Xi(\tilde{\mG}) \|\tilde{\Delta}_k\|_F + \frac{2 \lambda_n (2+\epsilon)}{K-1}\sum_{k'=1}^{K-1} \sqrt{2R_{k'}}  \|\tilde{\Delta}_k\|  . \label{ineq:basic2}
\end{align}

\paragraph{Step 3.} Bound below the LHS of the basic inequality.

%Substitute it into \eqref{ineq:basic2}, we have
%\begin{align}
%	\frac{1}{6n} \| \mathfrak{X}(\Delta_k)\|^2 -\frac{8}{n}(K-1)^2 C_k^2 \Xi^2(\tilde{\mG}) \leq \frac{4(K-1)C_K}{n}\Xi(\tilde{\mG}) \|\mathfrak{X}(\tilde{\Delta}_k)\| + \frac{\lambda_n K}{K-1}\sum_{k'=1}^{K-1} \sqrt{R_{k'}}  \|\tilde{\Delta}_k\|. \label{ineq:basic3}
%\end{align}

We apply the matrix completion theory developed by \cite{candes2010matrix} and \cite{candes2012exact}.
Let $k' \in \{1,\ldots,K\}$ be the index satisfying Assumption \ref{asmp:incoherence_convex}.
Since the value of the $L^2$-norm and the Frobenius norm is invariant to the shape of tensors, we compare the value of $\|Q_{k'}(\Delta_k)\|_F$ and $\|\tilde{\mathfrak{X}}(Q_{k'}(\Delta_k))\|_F$ with $k'$ instead of $\|\Delta_k\|_F$ and $\|\mathfrak{X}(\Delta_k)\|_F$.

For the matrix $Q_{k'}(X^*)$, we apply Assumption \ref{asmp:incoherence_convex} and obtain that $Q_{k'}(X^*)$ has the $\mu_{k'}$-incoherence property.
%Also, we obtain that $Q_{k'}( G_k^* \times_2 \tilde{G}_{<k} \times_3 \tilde{G}_{k<})$ satisfies $\tilde{\mu}$-incoherent property with $\tilde{\mu} = *****$.
Then, we apply Theorem 2 and Theorem 7 in \cite{candes2010matrix}, we obtain the following inequality as
\begin{align*}
	\|Q_{k'}(\tilde{\Delta}_{k})\|_F \leq \left(  \sqrt{\frac{48  \min\{I_{\leq k'}, I_{k'<}\}}{n}}+ 1 \right) \|\mathfrak{X}(Q_{k'}(\tilde{\Delta}_{k}))\|,
\end{align*}
with probability at least $1-( \max\{I_{\leq k'}, I_{k'<}\})^{-3}$ and
\begin{align*}
	n \geq C_{m} \mu_{k'}^2  \max\{I_{\leq k'}, I_{k'<}\} R_{k'} \log^3  \max\{I_{\leq k'}, I_{k'<}\},
\end{align*}
with a constant $C_m > 0$.
Then we obtain that
\begin{align*}
	&\frac{1}{n}\|\mathfrak{X}(\tilde{\Delta}_k)\|^2 = \frac{1}{n}\|\tilde{\mathfrak{X}}(Q_{k'}(\tilde{\Delta}_k))\|^2 \\
	& \geq (144  \min\{I_{\leq k'}, I_{k'<}\} + 3n)^{-1} \|Q_{k'}(\tilde{\Delta}_k)\|_F^2 =: C_{\kappa} \|Q_{k'}(\tilde{\Delta}_k)\|_F^2 = C_{\kappa} \|\tilde{\Delta}_k\|_F^2,
\end{align*}
where $C_{\kappa} > 0$ since $n \leq \prod_k I_k$.
Using this result into \eqref{ineq:basic2}, we have
\begin{align*}
	\frac{C_{\kappa}}{6} \|\tilde{\Delta}_k\|_F^2 \leq \frac{3(K-1)C_K}{n}\Xi(\tilde{\mG}) \|\tilde{\Delta}_k\|_F + \frac{2 \lambda_n (2+\epsilon)}{K-1}\sum_{k'=1}^{K-1} \sqrt{2R_{k'}}  \|\tilde{\Delta}_k\|_F.
\end{align*}
Then we obtain the inequality
\begin{align*}
	\frac{C_{\kappa}}{6} \|\tilde{\Delta}_k\|_F^2 \leq \frac{3(K-1)C_K}{n}\Xi(\tilde{\mG}) \|\tilde{\Delta}_k\|_F + \frac{2 \lambda_n (2+\epsilon) }{K-1}\sum_{k'=1}^{K-1} \sqrt{2R_{k'}}  \|\tilde{\Delta}_k\|_F.
\end{align*}
We divide the both hands side by $\|\tilde{\Delta}_k\|_F$ about the first term, and consider the root about the second term, then we have
\begin{align}
	\frac{C_{\kappa}}{6} \|\tilde{\Delta}_k\|_F & \leq \frac{3(K-1)C_K}{n}  \Xi(\tilde{\mG}) + \frac{2 \lambda_n (2+\epsilon)}{K-1}\sum_{k'=1}^{K-1} \sqrt{2R_{k'}},  \label{ineq:basic5}
\end{align}
by using the property $\|\mathfrak{X}(X)\| \leq \|X\|$ for all $X \in \mX$.

Finally, we define
\begin{align*}
	\Delta_k := (\hat{G}_k - G_k^*) \times_2 G_{<k}^* \times_3 G_{k<}^*,
\end{align*}
and compare $\Delta_k$ and $\tilde{\Delta}_k$ as
\begin{align*}
	\|\Delta_k\| \leq \|\tilde{\Delta}_k\| + \|\tilde{\Delta}_k - \Delta_k\|.
\end{align*}
We evaluate the last term by the same way of the step 2 as
\begin{align*}
	&\|\tilde{\Delta}_k - \Delta_k\|_F \\
	& \leq \left\| (\hat{G}_k - G_k^*) \times_2 (G_{<k}^* \times_3 G_{k<}^* - \tilde{G}_{<k} \times_3 \tilde{G}_{k<})\right\|_F \\
	& \leq 2 c_k \left\{\left\| (G_{<k}^* - \tilde{G}_{<k} ) \times_3 G_{k<}^*\right\|_F + \left\| \tilde{G}_{<k} \times_3 ( G_{k<}^* -  \tilde{G}_{k<})\right\|_F  \right\} \\
	& \leq 2 (K-1) C_K \Xi(\tilde{\mG}).
\end{align*}
Then, we have
\begin{align*}
	\|\Delta_k\|_F - 2 (K-1) C_k\Xi(\tilde{\mG}) \leq  \|\tilde{\Delta}_k\|_F.
\end{align*}
Substituting the result into \eqref{ineq:basic5}, we obtain
\begin{align}
	&\frac{C_{\kappa}}{6} \|\Delta_k\|_F \leq   2(K-1)C_K(1+n^{-1})  \Xi(\tilde{\mG})  + \frac{2 \lambda_n (2+\epsilon)}{K-1}\sum_{k'=1}^{K-1} \sqrt{2R_{k'}}. \label{ineq:basic6}
\end{align}

\paragraph{Step 4.} Combining the results.

Substituting the result of the step 3 into \eqref{ineq:basic6}, we finally obtain
\begin{align*}
	&\|\hat{G}_k - G_k^*\| \leq  6 (C_{\kappa} C_K)^{-1} 2(K-1)C_K(1+n^{-1})  \Xi(\tilde{\mG})  + \frac{2 \lambda_n (2+\epsilon)}{K-1}\sum_{k'=1}^{K-1} \sqrt{2R_{k'}}.
\end{align*}

\end{proof}

We back to the proof of Theorem \ref{thm:als}.
Based on the result of Lemma \ref{lem:als_1}, we will take two steps: (a) evaluate the distance between $X(\tilde{\mG})$ and $X(\mG^*)$, and (b) show the convergence as the ALS iteration proceeds.

\paragraph{Step (a).} Evaluate the distance between $X(\tilde \mG)$ and $X(\mG^*)$.

For brevity, we introduce new notation for $X(\mG)$.
Using the tensor product, we denote
\begin{align*}
	X(\mG) = G_1 \times_2 G_2 \times_3 \cdots \times_{K-1} G_{K-1} \times_K G_K.
\end{align*}
Then, we evaluate the distance between $X(\mG)$ and $X(\mG^*)$ as
\begin{align*}
	&X(\tilde\mG) - X(\mG^*)\\
	&=\tilde G_1 \times_2 \cdots \times_K \tilde G_K - G_1^* \times_2 \cdots \times_K G_K^* \\
	&=(\tilde G_1 \times_2 \cdots \times_{K-1} \tilde G_{K-1}\times_K \tilde G_K - \tilde G_1 \times_2 \cdots \times_{K-1} \tilde G_{K-1} \times_K G_K^*) \\
	& \quad + (\tilde G_1 \times_2 \cdots \times_{K-1} \tilde G_{K-1} \times_K G_K^* - \tilde G_1 \times_2 \cdots \times_{K-1} G_{K-1}^* \times_K G_K^*) \\
	& \cdots \\
	& \quad +  (\tilde G_1 \times_2 G_2^* \times _3\cdots \times_{K-1} G_{K-1}^* \times_K G_K^* - G_1^* \times_2 \cdots \times_K G_K^*) \\
	&= \sum_{k=1}^{K}  \tilde G_{<k} \times_{k} (\tilde G_{k} -  G_{k}^*) \times_{k+1}   G_{k<}^*.
\end{align*}

Then, we consider the Frobenius norm as
\begin{align*}
	&\|X(\tilde\mG) - X(\mG^*)\|_F^2 \\
	&= \left\| \sum_{k=1}^{K}  \tilde G_{<k} \times_{k} (\tilde G_{k} -  G_{k}^*) \times_{k+1}   G_{k<}^* \right\|_F^2 \\
	&= \sum_{k=1}^{K} \sum_{k'=1}^{K} \left\langle  \tilde G_{<k} \times_{k} (\tilde G_{k} -  G_{k}^*) \times_{k+1}   G_{k<}^*,  \tilde G_{<k'} \times_{k'} (\tilde G_{k'} -  G_{k'}^*) \times_{k'+1}   G_{k'<}^* \right\rangle \\
	&= \sum_{k=1}^K  \left\|   \tilde G_{<k} \times_{k} (\tilde G_{k} -  G_{k}^*) \times_{k+1}   G_{k<}^* \right\|_F^2 \\
	&\quad + \sum_{k=1}^{K} \sum_{k'\neq k} \left\langle  \tilde G_{<k} \times_{k} (\tilde G_{k} -  G_{k}^*) \times_{k+1}   G_{k<}^*,  \tilde G_{<k'} \times_{k'} (\tilde G_{k'} -  G_{k'}^*) \times_{k'+1}   G_{k'<}^* \right\rangle.
\end{align*}
As same as the proof of Lemma \ref{lem:als_1}, we bound the first term as
\begin{align*}
	\left\|   \tilde G_{<k} \times_{k} (\tilde G_{k} -  G_{k}^*) \times_{k+1}  \tilde G_{k<} \right\|_F^2 \leq C_K^2\Xi^2(\tilde{\mG}).
\end{align*}
For the second term, we obtain
\begin{align*}
	&\left \langle  \tilde G_{<k} \times_{k} (\tilde G_{k} -  G_{k}^*) \times_{k+1}  \tilde G_{k<},  \tilde G_{<k'} \times_{k'} (\tilde G_{k'} -  G_{k'}^*) \times_{k'+1}  \tilde G_{k'<} \right\rangle \\
	&\leq \left\| \tilde G_{<k} \times_{k} (\tilde G_{k} -  G_{k}^*) \times_{k+1}  \tilde G_{k<}\right\|_F \left\| \tilde G_{<k'} \times_{k'} (\tilde G_{k'} -  G_{k'}^*) \times_{k'+1}  \tilde G_{k'<} \right\|_F \\
	&\leq  C_K^2\Xi^2(\tilde{\mG}).
\end{align*}
Combining the results, we obtain 
\begin{align*}
	\|X(\tilde\mG) - X(\mG^*)\|_F^2 \leq (K+K^2)C_k^2 \Xi^2(\tilde{\mG}).
\end{align*}

\paragraph{Step (b).} Show convergence with the ALS iteration.

Let $\mG^{t}$ be a set $\mG$ obtained by $t$-th ALS iteration.
By the result of the step (a), we have
\begin{align*}
	\|X(\mG^t) - X(\mG^*)\|_F^2 \leq (K+K^2)C_K^2 \Xi(\mG^t).
\end{align*}
Applying the result of Lemma \ref{lem:als_1}, let $\hat{G}_k^t$ be the minimizer of optimization of \eqref{opt:als3} with the $t$-th ALS iteration, we obtain for each $t = 1,2,\ldots$,
\begin{align*}
	&\Xi(\mG^t) = \max_{k} \|\hat{G}_k^t - G_k^*\|_F \leq 6 (C_{\kappa} C_K)^{-1} 2(K-1)C_K(1+n^{-1})  \Xi(\tilde{\mG})  + \frac{2 \lambda_n (2+\epsilon)}{K-1}\sum_{k'=1}^{K-1} \sqrt{2R_{k'}}.
\end{align*}
The inequality holds since $\mG^{t-1}$ is the fixed $\tilde{\mG}$ for the $t$-th ALS iteration.
We define the contraction coefficient
\begin{align*}
	\chi := 12C_{\kappa}^{-1} (K-1)C_K(1.5+n^{-1}),
\end{align*}
and using the assumption that $\chi < 1$, we have
\begin{align}
	\Xi(\mG^t) \leq \max \left\{ \chi^t \Xi(\mG^0), 6 (C_{\kappa} C_K)^{-1}   \frac{2 \lambda_n (2 + \epsilon)}{K}\sum_{k'=1}^{K-1} \sqrt{2R_{k'}} \right\}, \label{ineq:xi}
\end{align}
where $\mG^0$ is an initial value of $\mG$.
With Assumption \ref{asmp:intial}, we set $t$ sufficiently large as 
\begin{align*}
	t \geq (\log \chi)^{-1} \left\{\log \left( 6 (C_{\kappa} C_K)^{-1}  \left( \frac{2 \lambda_n (2 + \epsilon)}{K}\sum_{k'=1}^{K-1} \sqrt{2R_{k'}}\right) \right) - \log \Xi(\mG^0)\right\},
\end{align*}
we obtain
\begin{align*}
	\|X(\mG^t) - X(\mG^*)\|_F^2 \leq 12 C_{\kappa}^{-1} K^2C_K \left( \frac{2 \lambda_n (2 + \epsilon)}{K}\sum_{k'=1}^{K-1} \sqrt{2R_{k'}} \right)^2.
\end{align*}
As we set $\hat{X} := X(\mG^t)$, we obtain the result.

\qed

\section{Time Complexity of TT-RALS}

To update $g_k^{(\ell + 1)}$, we need to compute
\begin{itemize}
\item $A = \Omega^T \Omega$, which requires $O(nI^2R^4)$,
\item $B = \sum_{k'=1}^{K-1} \Gamma_{k'}^{T}\Gamma_{k'}$, which requires
  $O(KD^2I^2R^4)$,
\item the inversion of an $IR^2 \times IR^2$ matrix
  $(A+B)$, which requires $O(I^3R^6)$,
\item $c=\Omega^T Y$, which requires $O(nIR^2)$,
\item $d=\tilde{V}_k(W_{k'}^{(\ell)})$, which requires $O(D^2)$,
\item
  $e=\frac{1}{K-1}\sum_{k'=1}^{K-1} \Gamma_{k'}^{T}( \eta d -
  \beta_{k'}^{(\ell)}$, which requires $O(KD^2IR^2)$,
\end{itemize}

To update $W_{k'}^{(\ell + 1)}$, we need to compute
\begin{itemize}
\item $a = \Gamma_{k'} g_k^{(\ell+1)}$, which requires $O(D^2IR^2)$,
\item $\tilde{V}_k^{-1} ( a +\beta_{k'}^{(\ell)} )$, which requires $O(D^2)$,
\item the proximal operation, which requires $O(D^3)$.
\end{itemize}

To update $\beta_{k'}^{(\ell + 1)}$, we need to compute
\begin{itemize}
\item $a = \Gamma_{k'} g_k^{(\ell+1)}$, which requires $O(D^2IR^2)$,
\item $b=\tilde{V}_k(W_{k'}^{(\ell + 1)})$, which requires $O(D^2)$,
\end{itemize}

Because there are $g_k^{(\ell + 1)}$ for $k=1,\dots,K$, 
$W_{k'}^{(\ell + 1)}$ for $k,k'=1,\dots,K$, and
$\beta_{k'}^{(\ell + 1)}$ for $k,k'=1,\dots,K$, the total time complexity is 

\begin{align*}
 &O(K(nI^2R^4 + KD^2I^2R^4 + I^3R^6 + nIR^2 + D^2 + KD^2IR^2))
\\
& \quad + O(K^2(D^2IR^2 + D^2 + D^3))+ O(K^2(D^2IR^2 + D^2))
\\
&= 
  O(K(nI^2R^4 + KD^2I^2R^4 + I^3R^6))
+ O(K^2(D^2IR^2 + D^3))
\\
&= 
  O(K(nI^2R^4 + KD^2I^2R^4 + I^3R^6))
+ O(K^2D^2IR^2)
\\
&= 
  O(nKI^2R^4 + K^2D^2I^2R^4 + KI^3R^6)
\\
&= 
  O(nKI^2R^4 + K^2D^2I^2R^4)
\end{align*}
In the third line and the last line, we assumed $D=O(IR^2)$ and $IR^2=O(n)$, respectively.

%%% Local Variables:
%%% mode: latex
%%% TeX-master: "TTcomp_NIPS2017.tex"
%%% End:

\bibliographystyle{abbrv}
\bibliography{tt}
\end{document}